\documentclass{article}
\PassOptionsToPackage{table}{xcolor}
\usepackage[T1]{fontenc}
\usepackage{iclr2027_conference,times}

\newif\ifpaperauthors
\paperauthorstrue
\ifpaperauthors
  \iclrfinalcopy
\fi

\usepackage{enumitem}
\usepackage{microtype}
\usepackage{pifont}
\usepackage{graphicx}
\usepackage{subcaption}
\usepackage{booktabs}
\usepackage{multirow}
\usepackage{xcolor}
\usepackage{colortbl}
\usepackage{xcolor}
\usepackage{hyperref}
\usepackage{amsmath}
\usepackage{amssymb}
\usepackage{mathtools}
\usepackage{amsthm}
\usepackage[capitalize,noabbrev]{cleveref}
\usepackage{caption}
\usepackage{algorithm}
\usepackage{algorithmic}
\usepackage[textsize=tiny]{todonotes}

\theoremstyle{plain}
\newtheorem{theorem}{Theorem}[section]
\newtheorem{proposition}[theorem]{Proposition}
\newtheorem{lemma}[theorem]{Lemma}

\theoremstyle{definition}
\newtheorem{definition}[theorem]{Definition}

\theoremstyle{remark}

\definecolor{markgreen}{RGB}{0,128,0}
\definecolor{markred}{RGB}{180,0,0}
\definecolor{markorange}{RGB}{220,130,0}
\newcommand{\cmark}{\textcolor{markgreen}{\textbf{\checkmark}}}
\newcommand{\xmark}{\textcolor{markred}{\textbf{$\times$}}}
\newcommand{\tmark}{\textcolor{markorange}{\textbf{$\triangle$}}}

\title{\(\pi\)-StepNFT: Wider Space Needs Finer Steps in Online RL for Flow-based VLAs}

\ifpaperauthors
\author{%
\begin{minipage}[t]{\dimexpr\textwidth-2\tabcolsep\relax}
\centering
\textbf{Siting Wang}\textsuperscript{1,2,3}\quad
\textbf{Xiaofeng Wang}\textsuperscript{1,4}\quad
\textbf{Zheng Zhu}\textsuperscript{1,\ensuremath{\dagger}}\quad
\textbf{Minnan Pei}\textsuperscript{2,3}\quad
\textbf{Xinyu Cui}\textsuperscript{2,3,5}\\[3pt]
\textbf{Cheng Deng}\textsuperscript{6}\quad
\textbf{Jian Zhao}\textsuperscript{5}\quad
\textbf{Guan Huang}\textsuperscript{1}\quad
\textbf{Haifeng Zhang}\textsuperscript{2,3}\quad
\textbf{Jun Wang}\textsuperscript{7,\ensuremath{\dagger}}\\[6pt]
\normalfont\small
\textsuperscript{1}GigaAI\quad
\textsuperscript{2}Institute of Automation, Chinese Academy of Sciences\\
\textsuperscript{3}University of Chinese Academy of Sciences\quad
\textsuperscript{4}Tsinghua University\\
\textsuperscript{5}Zhongguancun Academy\quad
\textsuperscript{6}The University of Edinburgh\quad
\textsuperscript{7}University College London\\[3pt]
\textsuperscript{\ensuremath{\dagger}}Corresponding authors. Contact: Zheng Zhu (\texttt{zhengzhu@ieee.org}).
\end{minipage}%
}
\hypersetup{pdfauthor={Siting Wang, Xiaofeng Wang, Zheng Zhu, Minnan Pei, Xinyu Cui, Cheng Deng, Jian Zhao, Guan Huang, Haifeng Zhang, Jun Wang}}
\else
\author{Anonymous authors}
\hypersetup{pdfauthor={}}
\fi

\begin{document}
\maketitle
\ifpaperauthors
  \lhead{Preprint}
\fi
\begin{abstract}

Flow-based vision-language-action (VLA) models have emerged as a powerful paradigm for embodied control, but multi-step denoising makes policy-gradient online reinforcement learning (RL) difficult because the final-action likelihood induced by the sampler is not directly tractable.
We identify an exploration--supervision mismatch within the denoising trajectory, where stochastic differential equation (SDE) sampling explores a wider local stochastic transition space while supervision derived from the final denoised output does not directly use the adjacent transitions observed during sampling.
We propose \textbf{\textit{$\boldsymbol{\pi}$-StepNFT}} (Step-wise Negative-aware Fine-Tuning), a flow-native online RL framework that couples wider stochastic exploration with finer step-wise supervision without computing the final-action likelihood.
$\pi$-StepNFT treats observed SDE transitions as supervision units, using variance-normalized transition errors to rank mirrored positive and negative velocity branches and map rollout feedback to local velocity updates.
On LIBERO, starting from shared few-shot SFT initializations, $\pi$-StepNFT raises average success from 57.6\% to 92.0\% for $\pi_0$ and from 77.1\% to 94.0\% for $\pi_{0.5}$. On ManiSkill, it achieves the highest reported mean success among the evaluated RL baselines on five test-only visual perturbations.
\ifpaperauthors
Our project homepage is available at
\href{https://wangst0181.github.io/pi-StepNFT/}{https://wangst0181.github.io/pi-StepNFT/}.
\fi
\end{abstract}

\section{Introduction}
\label{introduction}
Vision-language-action (VLA) models have recently shown increasingly general capabilities, enabling robots to follow open-ended natural language instructions and perform complex tasks across diverse environments. Early approaches either represented actions as discrete tokens~\citep{rt2, openvla} or mapped observations to continuous regression targets~\citep{openvlaoft}. Recent work increasingly adopts flow-matching policies~\citep{pi0, pi0.5, pi0.6, gigabrain, gr00t}, which combine large-scale vision-language pretraining with generative action prediction and have become a strong paradigm for complex manipulation. Yet the multi-step generative mechanism that gives flow-based VLAs flexibility also makes online policy optimization substantially harder than for explicit action-distribution policies.

After supervised fine-tuning (SFT) converges, robustness remains constrained by the behavioral manifold covered by expert demonstrations. Online reinforcement learning (RL) can use execution-time feedback to adapt the policy beyond these demonstrations. A distinctive challenge for flow-based VLAs is translating rollout feedback into generative-policy updates. The policy outputs a velocity field, and the executed action is obtained by numerically integrating an ordinary differential equation (ODE) over multiple denoising steps rather than by sampling once from an explicit action distribution. Consequently, exact likelihood-based updates require costly Jacobian-trace estimation, while direct optimization through the sampler entails solver backpropagation.

Online RL for flow-based VLAs involves two distinct temporal levels. At the environment level, temporal credit determines how feedback is assigned across states and executed actions according to their contribution to future returns. At the denoising level, rollout feedback must be translated into updates of the sequence of velocity predictions used to generate each action. Existing methods translate rollout feedback into policy improvement through objectives that either bypass or approximate the action likelihoods of generative policies. Some bypass explicit likelihoods through value-guided latent-space policy optimization~\citep{grrl,dsrl} or by conditioning action generation on an auxiliary value estimate~\citep{pi0.6}. Policy-gradient approaches instead approximate step-wise likelihoods through Gaussian parameterization along the denoising process~\citep{pirl}, allowing reward or advantage signals to be propagated through the sampler. In contrast, Diffusion-NFT~\citep{diffusionnft} provides a likelihood-free alternative from image generation, directly optimizing the flow field on the forward process and defining an implicit policy improvement direction through mirrored positive and negative branches.

Diffusion-NFT deliberately decouples optimization from the rollout sampler~\citep{diffusionnft}. This sampler-agnostic formulation is natural for image generation, where optimization focuses on a static final sample. In embodied control, however, perturbations along the denoising path can alter the executed action and propagate through closed-loop interaction. With the initial noise fixed, an ODE rollout follows a deterministic denoising path, whereas an equivalent stochastic differential equation (SDE) preserves the same marginals while exploring a local stochastic next-state distribution at each solver state $x_t$, yielding observed transitions $x_t\to x_{t^-}$~\citep{flowgrpo}. Diffusion-NFT still constructs intermediate flow-matching supervision from the final denoised output $x_0$ through the forward process, rather than directly using these observed transitions. This creates an exploration--supervision mismatch within the denoising trajectory.

Figure~\ref{fig:teaser} illustrates this mismatch and the design progression toward $\pi$-StepNFT. Under ODE rollouts, supervision derived from the final denoised output $x_0$ is applied at reconstructed interpolation states $x_t'$~\ding{182}. SDE stochasticity widens the local transition space~\ding{183} and introduces timestep-dependent transition noise. We move supervision to sampled solver states $x_t$ with noise-aware normalization~\ding{184}, but targets derived from $x_0$ remain coarse. $\pi$-StepNFT instead uses the observed next state $x_{t^-}$~\ding{185}, yielding finer step-wise guidance over the wider stochastic transition space.

\begin{figure*}[t]
    \centering
    \includegraphics[width=\textwidth]{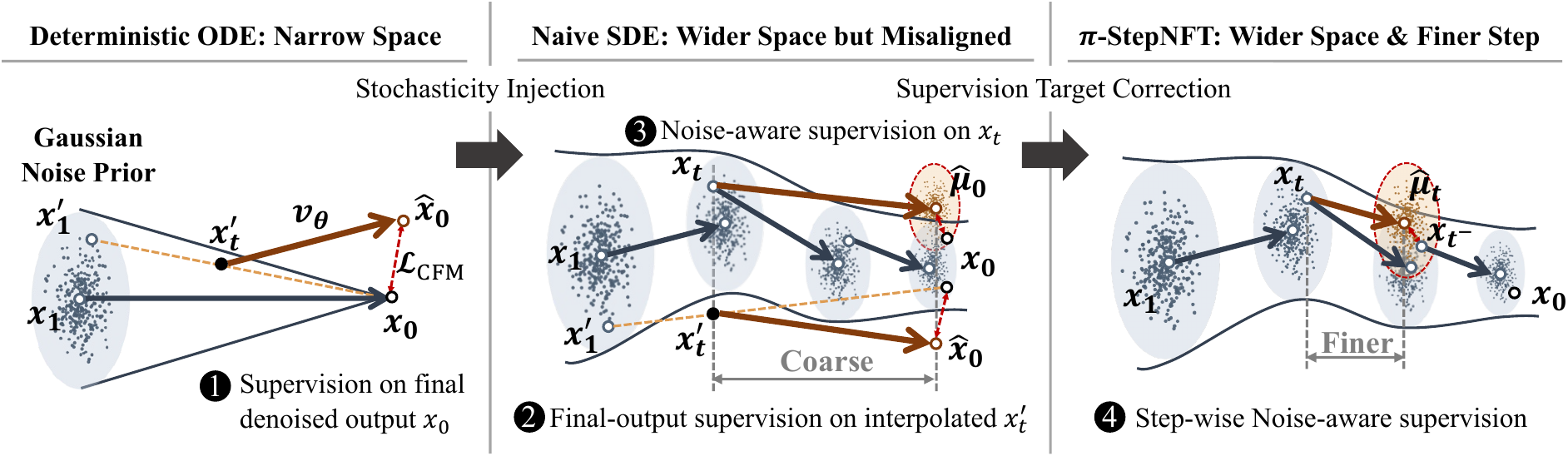}
    \vspace{-15pt}
    \caption{\textbf{Illustration of the design of $\pi$-StepNFT.}
    ODE~\ding{182} and naive SDE~\ding{183} retain supervision from the final denoised output $x_0$ at reconstructed interpolation states $x'_t$.
    Noise-aware supervision~\ding{184} moves to sampled solver states $x_t$ with local covariance normalization.
    $\pi$-StepNFT~\ding{185} further targets the observed next state $x_{t^-}$.
    Wider space denotes the local stochastic transition neighborhood.}
    \label{fig:teaser}
    \vspace{-15pt}
\end{figure*}

To address this mismatch, we propose \textbf{\textit{$\boldsymbol{\pi}$-StepNFT}}, a flow-native online RL framework for flow-based VLAs coupling wider stochastic exploration with finer step-wise supervision without computing final-action likelihoods. It uses variance-normalized errors on each observed SDE transition and replaces regression weighting with contrastive ranking over mirrored positive and negative velocity branches to map rollout feedback to local velocity updates.
Our contributions are as follows:

\begin{itemize}[leftmargin=1.2em]
    \item We identify an exploration--supervision mismatch in embodied generative control. SDE sampling explores a wider local stochastic transition space, while supervision derived from the final denoised output does not directly use the adjacent transitions observed during SDE sampling.
    
    \item We propose a flow-native online RL objective that maps rollout feedback to local velocity updates on observed SDE transitions, using variance-normalized transition errors to rank mirrored positive and negative velocity branches.

    \item We show that the expected ranking gradient follows the outcome-conditioned transition-mean difference at the rollout policy. Relative branch comparison removes the fixed quadratic displacement penalty in weighted regression and, under binary feedback, explicitly optimizes the relative error between both mirrored branches.
    
    \item We validate $\pi$-StepNFT on LIBERO~\citep{libero} and ManiSkill~\citep{maniskill}. On LIBERO, starting from shared few-shot SFT initializations, $\pi$-StepNFT raises average success from 57.6\% to 92.0\% for $\pi_0$ and from 77.1\% to 94.0\% for $\pi_{0.5}$. On ManiSkill, it obtains the highest reported mean success among evaluated RL baselines on five test-only visual shifts.
\end{itemize}

\section{Preliminaries}
We use two time scales: environment steps and denoising time. An episode trajectory is indexed by environment steps $i=0,\dots,H-1$, yielding $\tau=\{(s_i,a_i)\}_{i=0}^{H-1}$. Separately, a flow policy generates each action through denoising over continuous time $t\in[0,1]$ or its discretization. When discretized, we use $K$ solver steps with schedule $1=t_0>t_1>\cdots>t_K=0$, and denote sampler states by $\{x_{t_j}\}_{j=0}^{K}$. Unless noted, $i$ indexes environment interaction and $t$ denotes denoising time.

\subsection{Flow Matching for VLA Models}
\label{Flow Matching for VLA Models}
We consider a flow-based VLA policy that maps context $c$ and Gaussian noise $x_1\sim\mathcal{N}(0,I)$ to an executed action $x_0\in\mathbb{R}^d$. Flow matching~\citep{flow_matching} learns a velocity field $v_\theta(x_t,t,c)$ over interpolants $x_t' = t x_1 + (1-t)x_0$, with target velocity $u_t = x_1 - x_0$. Since $x_0=x_t'-t u_t$, this velocity regression can also be interpreted as supervision on $x_0$ (Figure~\ref{fig:teaser} (Left)). The standard conditional flow-matching (CFM) objective is
\begin{equation}
\label{eq:CFM_loss}
   \mathcal{L}_{\mathrm{CFM}}(\theta)
    =
    \mathbb{E}_{t,x_0,x_1}
    \left[
    \left\|v_\theta(x_t',t,c)-u_t\right\|_2^2
    \right]. 
\end{equation}
\textbf{ODE-to-SDE Sampling.}
Deterministic flow sampling integrates the ODE from $t=1$ to $t=0$, so each $x_t$ maps to a single next state $x_{t^-}$. The reverse-time SDE~\citep{flowgrpo} preserves the same marginals but adds Wiener noise, turning this update into a Gaussian transition $q_{\theta,t}(x_{t^-}\mid x_t,c)=\mathcal{N}(\mu_{\theta,t}(x_t,c),\Sigma_t)$ with $\Sigma_t=\sigma_t^2\delta_t I$, where $\sigma_t=\sigma_0\sqrt{\frac{t}{1-t^-}}$, thereby widening the local transition neighborhood. Its mean remains affine in the predicted velocity:
\begin{equation}
    \label{eq:affine_mean}
    \mu_{\theta,t}(x_t,c)=U_t(x_t,t)+B_t(t)\,v_\theta(x_t,t,c),
\end{equation}
where $U_t$ and $B_t$ depend on the noise schedule (Appendix~\ref{app:proof_Affine Mean Derivation}).

\subsection{Negative-aware Fine-tuning}
\label{sec:nft_prelim}
Negative-aware Fine-tuning (NFT) was introduced for verification-driven math reasoning as a supervised-style alternative to policy-gradient RL~\citep{nft}. Given online generations from a rollout policy $\pi^{\mathrm{old}}$, a binary verifier partitions them into positive and negative sets $\mathcal{D}^+$ and $\mathcal{D}^-$. NFT treats them as samples from two latent policies $\pi^+$ and $\pi^-$, using their gap to define an implicit improvement direction instead of discarding failures as in rejection-sampling fine-tuning.

Diffusion-NFT~\citep{diffusionnft} extends this idea to flow-based image generation, where the policy is induced by a velocity field. Let $v^{\mathrm{old}}$ be the rollout field, and let $v^\pm$ be the latent fields associated with positive and negative samples. Since they are unobserved, Diffusion-NFT uses mirrored branches around $v^{\mathrm{old}}$,
\begin{equation}
    \label{eq:mirror}
    v^+_\theta = (1-\beta)v^{\mathrm{old}}+\beta v_\theta,
    \qquad
    v^-_\theta = (1+\beta)v^{\mathrm{old}}-\beta v_\theta,
\end{equation}
and optimizes the weighted flow-matching loss
\begin{equation}
    \label{eq:nft_loss}
    \mathcal{L}_{\mathrm{NFT}}(\theta)
    =
    r\,\mathcal{L}^+_{\mathrm{CFM}}(\theta)
    +(1-r)\,\mathcal{L}^-_{\mathrm{CFM}}(\theta),
\end{equation}
where $\mathcal{L}^{\pm}_{\mathrm{CFM}}$ replaces the predictor in Eq.~\eqref{eq:CFM_loss} with $v^\pm_\theta$. This preserves the supervised regression structure of flow matching while linking the optimum to the oracle improvement direction: under unlimited data and model capacity, $v_\theta^\star = v^{\mathrm{old}} + \frac{2}{\beta}\Delta$, where $\Delta$ denotes the latent positive direction.

\section{Method}

We introduce $\pi$-StepNFT (Step-wise Negative-aware Fine-Tuning), a flow-native online RL framework for flow-based VLAs that avoids computing the final-action likelihood induced by the full sampler. As illustrated in Figure~\ref{fig:teaser}, rather than reconstructing flow-matching supervision from the final output $x_0$, we use each sampled SDE transition $x_t\!\rightarrow\!x_{t^-}$ as a local supervision unit and normalize its transition errors by the local SDE covariance.

We further reformulate how optimality feedback enters negative-aware fine-tuning.
Eq.~\eqref{eq:nft_loss} uses the optimality score as a regression weight, causing the two mirrored branches to contribute in different proportions.
While this weighting is natural for graded final-sample quality in image generation, it does not explicitly compare the two mirrored branch errors.
We instead use the optimality score to define a branch preference and optimize a relative ranking objective, encouraging the outcome-favored branch to better explain the same observed transition.

{\setlength{\textfloatsep}{4pt plus 1pt minus 2pt}
\begin{algorithm}[t]
  \caption{$\pi$-StepNFT: Step-wise Negative-aware Fine-Tuning}
  \label{alg:stepnft}
  \small
  \setlength{\algorithmicindent}{0.9em}
  \begin{algorithmic}[1]
    \STATE {\bfseries Require:} SFT policy $\pi^{\mathrm{old}}$ induced by $v^{\mathrm{old}}$, simulator $\mathcal{E}$, schedule $\{t_j\}_{j=0}^K$, hyperparameters $\beta,\lambda_{\mathrm{TR}}$.
    \STATE {\bfseries Initialize:} trainable velocity $v_\theta\leftarrow v^ {\mathrm{old}}$ and buffer $\mathcal{D}\leftarrow\emptyset$; use shorthand $x_{i,j}:=x_{t_j}^{(i)}$.

    \FOR{each iteration $m$}
      \STATE \textcolor{gray}{// Rollout Phase}
      \FOR{each task $(s_0,l)$}
        \STATE Rollout $\pi^{\mathrm{old}}$ in $\mathcal{E}$ to obtain trajectory $\tau$ and SDE paths $\{\{x_{t_j}^{(i)}\}_{j=0}^{K}\}_{i=0}^{H-1}$.
        \STATE Assign an optimality score $r_i\in[0,1]$ to each environment step $i$ from rollout feedback.
        \STATE Store $\mathcal{D}\leftarrow\mathcal{D}\cup\{(x_{t_j}^{(i)},x_{t_{j+1}}^{(i)},v^{\mathrm{old}}(x_{t_j}^{(i)},t_j,c_i),t_j,c_i,r_i)\text{ where } j\sim\mathcal{U}\{0,\ldots,K-1\}\}_{i=0}^{H-1}$.
      \ENDFOR

      \STATE \textcolor{gray}{// Optimization Phase}
      \FOR{each mini-batch $(x_t,x_{t^-},v_t^{\mathrm{old}},t,c,r)\sim\mathcal{D}$}
        \STATE $\Delta v_{\theta,t}\leftarrow v_\theta(x_t,t,c)-v_t^{\mathrm{old}}, \quad v^{\pm}_{\theta,t}\leftarrow v_t^{\mathrm{old}}\pm\beta\Delta v_{\theta,t}$.
        \STATE $(\mu^{\pm}_{\theta,t},\Sigma_t)\leftarrow\textsc{MeanCov}(x_t,v^{\pm}_{\theta,t},t)$, \quad $E^{\pm}_{\theta,t}\leftarrow\|x_{t^-}-\mu^{\pm}\|_{\Sigma_t^{-1}}^2$.\hfill $\triangleright$ \textsc{MeanCov}: Eq.~\eqref{eq:affine_mean}
        \STATE Update policy: $\theta\leftarrow\theta-\eta\nabla_\theta\left[\operatorname{softplus}\left(\frac{1}{2}(2r-1)(E^{+}_{\theta,t} - E^{-}_{\theta,t})\right)+\lambda_{\mathrm{TR}}\|\Delta v_{\theta,t}\|_2^2\right]$.
      \ENDFOR

      \STATE EMA update rollout policy: $\theta^{\mathrm{old}}\leftarrow\alpha_m\theta^{\mathrm{old}}+(1-\alpha_m)\theta$, \quad $\mathcal{D}\leftarrow\emptyset$.
    \ENDFOR
    \STATE {\bfseries Output:} optimized flow policy $\pi_\theta$ induced by $v_\theta$.
  \end{algorithmic}
\end{algorithm}}

\subsection{Step-wise Transitions and Mirror Errors} 
\label{sec:definition}
We collect rollouts with an EMA policy $\pi^{\mathrm{old}}$, where $\pi_\theta$ is the action policy induced by velocity field $v_\theta$. At environment step $i$, the SDE solver applies $v^{\mathrm{old}}$ over denoising steps to produce $\{x^{(i)}_{t_j}\}_{j=0}^{K}$ (omitting $(i)$ below), then executes chunked $x_{t_K}$, equivalently $x_0$, as the action.
We uniformly sample one adjacent transition $(x_t,x_{t^-},t)=(x_{t_j},x_{t_{j+1}},t_j)$ with $j\sim\mathcal{U}\{0,\ldots,K-1\}$ and cache $v_t^{\mathrm{old}}=v^{\mathrm{old}}(x_t,t,c)$.

Given current prediction $v_{\theta,t}=v_{\theta}(x_t,t,c)$, define its local deviation from the rollout policy:
$\Delta v_{\theta,t}=v_{\theta,t}-v_t^{\mathrm{old}}$.
Analogous to Eq.~\eqref{eq:mirror}, we form two local velocity candidates around the rollout velocity:
$v_{\theta,t}^{\pm}=v_t^{\mathrm{old}}\pm \beta \Delta v_{\theta,t}$,
where $\beta>0$ controls branch displacement. This is equivalent to Eq.~\eqref{eq:mirror}, but applied to one observed solver transition.
By Eq.~\eqref{eq:affine_mean}, candidate $v_{\theta,t}^{\pm}$ yields a one-step Gaussian transition with shared covariance $\Sigma_t$ and mean
$\mu_{\theta,t}^{\pm}=\mu_t(v_{\theta,t}^{\pm})$.
We measure how well each branch explains $x_{t^-}$ by variance-normalized transition errors
\begin{equation}
\label{eq:transition}
    E_{\theta,t}^{\pm}
    =
    \left\|x_{t^-}-\mu_{\theta,t}^{\pm}\right\|_{\Sigma_t^{-1}}^2 .
\end{equation}
The local SDE covariance $\Sigma_t$ normalizes transition errors for comparison across denoising steps.

\subsection{Step-wise Contrastive Objective}
\label{sec:objective}
Algorithm~\ref{alg:stepnft} assigns each sampled SDE transition the environment-step optimality score $r_i\in[0,1]$, denoted $r$ below. By default, $r_i=R(\tau)\in\{0,1\}$ broadcasts the trajectory outcome. Critic-derived feedback can provide step-dependent scores. The signed label $y=2r-1$ identifies the preferred branch, while $|y|$ controls preference strength. Ideally, we compare whether $x_{t^-}$ is more likely under $p(x_{t^-}\mid x_t,c,o=1)$ or $p(x_{t^-}\mid x_t,c,o=0)$. Here, $o\in\{0,1\}$ denotes the latent outcome underlying the positive and negative transition splits. Since these oracle outcome-conditioned distributions are unobservable, we use the mirrored branches from Section~\ref{sec:definition}.

\begin{definition}[$\pi$-StepNFT Step-wise Ranking Loss]
\label{def:stepnft_obj}
For a sampled transition tuple $(x_t,x_{t^-},t,c)$ with signed label $y\in[-1,1]$, let $E_{\theta,t}^{+}$ and $E_{\theta,t}^{-}$ be the variance-normalized errors from Section~\ref{sec:definition}. Since $q_{\theta,t}^{\pm}(x_{t^-}\mid x_t,c)=\mathcal{N}(\mu_{\theta,t}^{\pm},\Sigma_t)$ share covariance, the step-wise contrastive loss is
\begin{equation}
    \ell_t(\theta) =
    \operatorname{softplus}
    \left(
        \frac{1}{2}y
        \left(E_{\theta,t}^{+}-E_{\theta,t}^{-}\right)
    \right) =
    \operatorname{softplus}
    \left(
    -y
    \log
    \frac{
    q_{\theta,t}^{+}(x_{t^-}\mid x_t,c)
    }{
    q_{\theta,t}^{-}(x_{t^-}\mid x_t,c)
    }
    \right).
\label{eq:step_ranking}
\end{equation}
\end{definition}

Minimizing Eq.~\eqref{eq:step_ranking} ranks the two mirrored branches on the same solver transition. For $y>0$, the loss favors a larger positive-to-negative transition log-ratio, or a smaller relative error $E_{\theta,t}^{+}-E_{\theta,t}^{-}$; $y<0$ reverses this preference. Unlike DPO-style preference learning over two sampled outputs, our objective ranks local perturbations around the rollout policy on one transition. Since $y$ supplies preference direction, optimization uses only the one-step transition ratio, not the final-action likelihood induced by the full sampler.

In practice, Algorithm~\ref{alg:stepnft} adds
$\lambda_{\mathrm{TR}}\|\Delta v_{\theta,t}\|_2^2$
to control update magnitude. Since this term stabilizes optimization but does not define branch preference, Section~\ref{sec:validity} analyzes the ranking loss $\ell_t(\theta)$.

\subsection{Objective Comparison with Weighted Regression}
\label{sec:weighted_regression_comparison}

To isolate the effect of the learning objective, we compare step-wise ranking with a weighted-regression baseline using the same observed SDE transition and the same variance-normalized mirrored errors $E_{\theta,t}^{\pm}$. Under our step-wise formulation, the regression baseline replaces the conditional flow-matching losses in Eq.~\eqref{eq:nft_loss} with these transition errors:
\begin{equation}
    \ell^{\mathrm{wMSE}}_t(\theta)
    =
    r E^+_{\theta,t} + (1-r)E^-_{\theta,t}.
    \label{eq:step_wmse}
\end{equation}
Although both objectives use the same mirrored branches and transition errors, they differ in how they treat branch displacement and combine the two branch errors.

\begin{lemma}[Weighted regression and step-wise ranking decomposition]
\label{lem:wmse_ranking_decomposition}
For a sampled solver transition $(x_t,x_{t^-},t,c)$, let
$\mu_t^{\mathrm{old}}=\mu_t(v_t^{\mathrm{old}})$ be the rollout transition mean,
$e_t=x_{t^-}-\mu_t^{\mathrm{old}}$ be the residual of the rollout mean,
$d_t=\mu^+_{\theta,t}-\mu_t^{\mathrm{old}}$ be the displacement of the positive branch in mean space,
and $B_t$ be the affine coefficient from Eq.~\eqref{eq:affine_mean}.
Under the mirrored construction in Section~\ref{sec:definition}, we have
$d_t = \beta B_t \Delta v_{\theta,t}$,
$\mu^\pm_{\theta,t}=\mu_t^{\mathrm{old}}\pm d_t$.
Therefore:

\textnormal{(a)} The step-wise weighted-regression objective in Eq.~\eqref{eq:step_wmse} satisfies
\begin{equation}
    \ell^{\mathrm{wMSE}}_t(\theta)
    =
    \mathrm{const}
    -2y\left\langle \Sigma_t^{-1}e_t,d_t\right\rangle
    +
    \left\|d_t\right\|^2_{\Sigma_t^{-1}}.
    \label{eq:wmse_decomposition}
\end{equation}
\textnormal{(b)} The step-wise ranking objective in Eq.~\eqref{eq:step_ranking} satisfies
\begin{equation}
    \ell_t(\theta)
    =
    \mathrm{softplus}
    \left(
    -2y
    \left\langle \Sigma_t^{-1}e_t,d_t\right\rangle
    \right).
    \label{eq:ranking_decomposition}
\end{equation}
\end{lemma}

\textbf{Fixed displacement penalty in weighted regression.}
Eq.~\eqref{eq:wmse_decomposition} decomposes weighted regression into the outcome-signed alignment term $-2y\langle\Sigma_t^{-1}e_t,d_t\rangle$ and a fixed quadratic displacement penalty $\|d_t\|_{\Sigma_t^{-1}}^2$. The latter is independent of the label $y$ and therefore always penalizes displacement from the rollout transition mean $\mu_t^{\mathrm{old}}$. In contrast, this common quadratic term cancels when the two mirrored branch errors are compared relatively in Eq.~\eqref{eq:ranking_decomposition}.

\textbf{Relative branch comparison under binary feedback.}
For $r=1$, weighted regression reduces to $E_{\theta,t}^{+}$, while for $r=0$ it reduces to $E_{\theta,t}^{-}$. Thus, each binary-labeled sample optimizes one outcome-selected absolute branch error. Ranking instead directly optimizes the relative error $E_{\theta,t}^{+}-E_{\theta,t}^{-}$, placing both mirrored branch errors in every sample objective. Both objectives remain negative-aware by learning from positive and negative samples, but ranking makes the branch comparison explicit at the per-sample level.

Together, ranking and weighted regression use the same mirrored branches and outcome feedback but differ in objective structure. Ranking removes the fixed quadratic displacement penalty in weighted regression and replaces outcome-selected absolute-error fitting with explicit relative branch comparison under binary feedback.

\subsection{\texorpdfstring{Outcome-Conditioned Gradient Analysis}{Outcome-Conditioned Gradient Analysis}}
\label{sec:validity}

We examine whether ranking the constructed mirrored branches yields an update related to outcome-conditioned differences in local denoising transitions. Fix context $c$, solver state $x_t$ and denoising time $t$. With latent outcome $o\in\{0,1\}$ (Section~\ref{sec:objective}), define outcome-conditioned transition means
\begin{equation}
    \mu_t^{+}(x_t,c)=\mathbb{E}[x_{t^-}\mid x_t,c,o=1],
    \qquad
    \mu_t^{-}(x_t,c)=\mathbb{E}[x_{t^-}\mid x_t,c,o=0].
\end{equation}
We denote their difference by $\Delta\mu_t^{\star}=\mu_t^{+}-\mu_t^{-}$,
the oracle outcome-conditioned transition-mean difference. These means are unavailable during training. $\pi$-StepNFT instead constructs mirrored transition means $\mu_{\theta,t}^{\pm}=\mu_t^{\mathrm{old}}\pm d_t$, where $d_t=\beta B_t\Delta v_{\theta,t}$. They are local perturbations around the rollout transition mean rather than estimates of $\mu_t^{\pm}$. We therefore ask whether ranking these constructed branches yields an expected update related to $\Delta\mu_t^{\star}$. The following proposition establishes this relation when the actor matches the rollout policy and accounts for soft-score mismatch.

\begin{proposition}[Outcome-conditioned gradient at the rollout policy]
\label{prop:local_alignment}
Fix the rollout transition distribution with $\Sigma_t\succ0$, and hold cached rollout quantities and scores constant with respect to $\theta$. Assume the conditional expectations below exist. Let $\theta_0$ satisfy $v_{\theta_0,t}=v_t^{\mathrm{old}}$, and set $J_0=\left.\partial v_{\theta,t}/\partial\theta\right|_{\theta=\theta_0}$.
For $o\in\{0,1\}$, $r\in[0,1]$, and $y=2r-1$, define
$\alpha(x_t,c)=\mathbb P(o=1\mid x_t,c)\in(0,1)$,
and
$\kappa_t(x_t,c)=2\alpha(x_t,c)\bigl(1-\alpha(x_t,c)\bigr)$.
Let
$\varepsilon_t^r(x_t,c)=2\mathbb E[(r-o)e_t\mid x_t,c]$
be the score-mismatch bias, with $e_t$ from Lemma~\ref{lem:wmse_ranking_decomposition}. Then
\begin{equation}
    \left.\mathbb{E}\!\left[-\nabla_{\theta}\ell_t(\theta)\mid x_t,c\right]\right|_{\theta=\theta_0}
    =
    \beta J_0^\top B_t^\top \Sigma_t^{-1}
    \left[
        \kappa_t(x_t,c)\Delta\mu_t^\star(x_t,c)
        +
        \varepsilon_t^r(x_t,c)
    \right],
    \label{eq:soft_score_alignment_main}
\end{equation}
When $\varepsilon_t^r(x_t,c)=0$ and $0<\alpha(x_t,c)<1$, this becomes
\begin{equation}
    \left.\mathbb{E}\!\left[-\nabla_{\theta}\ell_t(\theta)\mid x_t,c\right]\right|_{\theta=\theta_0}
    =\beta\kappa_t J_0^\top B_t^\top\Sigma_t^{-1}\Delta\mu_t^\star(x_t,c).
\end{equation}
\end{proposition}

For binary feedback $r=o$, the score-mismatch bias $\varepsilon_t^r$ vanishes, and the outcome-conditioned transition-mean difference enters the expected ranking gradient through covariance weighting, $B_t^\top$, and the velocity Jacobian. Thus, the ranking objective yields an outcome-conditioned local update without estimating the oracle means.
For soft feedback, $\varepsilon_t^r$ measures how score mismatch correlates with the transition residual. Posterior-like scores $r(x_0,c)=\mathbb P(o=1\mid x_0,c)$ give zero bias when the intermediate transition provides no additional outcome information beyond $x_0$ and $c$, i.e., when
$\mathbb P(o=1\mid x_0,c,x_t,x_{t^-})=\mathbb P(o=1\mid x_0,c)$.

For finite deviations from the rollout policy under binary feedback, Appendix~\ref{app:finite_update_alignment} derives an additional finite-update correction $-\zeta_t$. It vanishes when the actor matches the rollout policy and remains bounded. The analysis also gives a sufficient condition for positive alignment with the oracle transition-mean difference in variance-normalized transition-mean space. Its magnitude depends on the actor--rollout gap, which EMA rollout updates and trust-region regularization help limit.

\section{Experiments}
We evaluate $\pi$-StepNFT through gains over few-shot SFT, OOD robustness under visual shifts (Section~\ref{sec:main_results}), and mechanistic analysis of its design (Section~\ref{sec:diagnostic}).

\subsection{Experiment Setup}
\label{sec:experimental_setup}

\textbf{Benchmarks \& Baselines.} We evaluate $\pi$-StepNFT on LIBERO (Spatial, Object, Goal, Long) and ManiSkill (PutOnPlateInScene). We additionally use CALVIN~\citep{calvin} for long-horizon diagnostics. Using $\pi_0$ and $\pi_{0.5}$ from OpenPi as base policies, we initialize from $\pi_\texttt{RL}$~\mbox{\citep{pirl}} SFT checkpoints. On LIBERO, all online RL methods start from the same controlled low-demonstration SFT checkpoints, allowing us to measure policy improvement from a shared initialization; ManiSkill uses full multitask SFT.

\textbf{Training Protocol.} During RL fine-tuning with RLinf~\citep{rlinf}, we freeze the VLM backbone and update the flow-matching action expert. Unless otherwise specified, $\pi$-StepNFT uses trajectory-level binary success on LIBERO and ManiSkill. GRPO- and GAE-based soft feedback are mapped to $r\in[0,1]$ using Diffusion-NFT's normalization~\citep{diffusionnft} with fixed scale $Z=3$. $\pi$-StepNFT + Critic uses a PPO critic to compute GAE while retaining the same ranking objective for actor updates. We report success rates. Appendix~\ref{app:experiment_details} gives implementation details.

\subsection{Main Results: Online Improvement and Robustness under Visual Shifts}
\label{sec:main_results}

\begin{table}[t]
\centering
\begin{minipage}[t]{0.46\textwidth}
\vspace{-15pt}
\caption{Success rates (\%) on LIBERO. $\pi$-StepNFT: three-seed means. $\pi_\texttt{RL}$: reported results from \citet{pirl}.}
\centering
\scriptsize
\setlength{\tabcolsep}{3pt}
\renewcommand{\arraystretch}{1.08}
\resizebox{\linewidth}{!}{%
\begin{tabular}{llcccccc}
\toprule
\textbf{Model} & & \textbf{Spatial} & \textbf{Object} & \textbf{Goal} & \textbf{Long} & \textbf{Avg.} & $\Delta$ \textbf{Avg.} \\
\midrule
\rowcolor{gray!15}
\multicolumn{8}{l}{\textbf{\# Full SFT}} \\
\multicolumn{2}{l}{Octo~\citep{octo}}               & 78.9 & 85.7 & 84.6 & 51.1 & 75.1 & --- \\
\multicolumn{2}{l}{OpenVLA~\citep{openvla}}            & 84.7 & 88.4 & 79.2 & 53.7 & 76.5 & --- \\
\multicolumn{2}{l}{$\pi_{\text{fast}}$~\citep{fast}}& 96.4 & 96.8 & 88.6 & 60.2 & 85.5 & --- \\
\multicolumn{2}{l}{OpenVLA-OFT~\citep{openvlaoft}}        & 91.6 & 95.3 & 90.6 & 86.5 & 91.0 & --- \\
\multicolumn{2}{l}{$\pi_{0}$}          & 96.8 & 98.8 & 95.8 & 85.2 & 94.2 & --- \\
\multicolumn{2}{l}{$\pi_{0.5}$}        & 98.8 & 98.2 & 98.0 & 92.4 & 96.9 & --- \\
\midrule
\rowcolor{blue!8}
\multicolumn{8}{l}{\textbf{\# Few-shot SFT + RL}} \\
\multirow{5}{*}{$\pi_{0}$}
& SFT                 & 65.3 & 64.4 & 49.8 & 51.2 & 57.6 & --- \\
& $\pi_\texttt{RL}$ (Flow-SDE + PPO)      & 98.4 & 99.4 & 96.2 & 90.2 & 96.0 & +38.4 \\
& $\pi_\texttt{RL}$ (Flow-SDE + GRPO)      & 97.8 & 97.8 & 83.2 & 81.4 & 90.0 & +32.4 \\
& $\pi$-StepNFT       & 95.4 & 98.0 & 87.9 & 86.7 & 92.0 & +34.4 \\
& $\pi$-StepNFT + Critic & 97.6 & 98.0 & 89.8 & 90.0 & 93.9 & +36.3 \\
\midrule
\rowcolor{green!8}
\multicolumn{8}{l}{\textbf{\# Few-shot SFT + RL}} \\
\multirow{5}{*}{$\pi_{0.5}$}
& SFT                 & 84.6 & 95.4 & 84.6 & 43.9 & 77.1 & --- \\
& $\pi_\texttt{RL}$ (Flow-SDE + PPO)      & 99.6 & 100  & 98.8 & 93.0 & 97.9 & +20.8 \\
& $\pi_\texttt{RL}$ (Flow-SDE + GRPO)     & 97.4 & 99.8 & 91.2 & 77.6 & 91.5 & +14.4 \\
& $\pi$-StepNFT       & 97.8 & 100  & 98.2 & 79.8 & 94.0 & +16.9 \\
& $\pi$-StepNFT + Critic & 98.0 & 99.3 & 98.2 & 88.4 & 96.0 & +18.9 \\
\bottomrule
\end{tabular}
}
\label{tab:libero}
\end{minipage}
\hfill
\begin{minipage}[t]{0.52\textwidth}
    \vspace{-15pt}
\caption{Success rates (\%) on ManiSkill for $\pi_{0.5}$. $\pi$-StepNFT: three-seed means. $\pi_\texttt{RL}$: reported results from \citet{pirl}.}
\centering
\scriptsize
\setlength{\tabcolsep}{2pt}
\renewcommand{\arraystretch}{1.16}
\resizebox{\linewidth}{!}{%
\begin{tabular}{llcccc}
\toprule
\textbf{Environment} & \textbf{Variation} & $\pi_{0.5}$-SFT & \begin{tabular}[c]{@{}c@{}}$\pi_{0.5}$-RL\\Flow-SDE\end{tabular} & \begin{tabular}[c]{@{}c@{}}$\pi_{0.5}$-RL\\Flow-Noise\end{tabular} & $\pi_{0.5}$-StepNFT \\
\midrule
\multirow{1}{*}{In distribution}
& Main-v3-train & 40.1 & \textbf{90.9} & 89.7 & 84.3 \\
\midrule
\multirow{8}{*}{\begin{tabular}[c]{@{}l@{}}Visual-Language\\Variations\end{tabular}}
& Instruct-v1-test & 46.6 & 77.0 & \textbf{85.7} & 82.9 \\
& VisionImage-v1-test & 46.2 & 78.8 & \textbf{83.1} & 80.5 \\
& VisionTexture03-v1-test & 36.7 & 69.6 & 75.0 & \textbf{80.5} \\
& VisionTexture05-v1-test & 32.7 & 58.0 & 62.2 & \textbf{71.2} \\
& VisionWhole03-v1-test & 40.1 & 69.6 & 71.6 & \textbf{82.0} \\
& VisionWhole05-v1-test & 30.7 & 55.0 & 57.0 & \textbf{62.5} \\
\cmidrule{2-6}
& \textit{Visual Avg. (5)} & 37.3 & 66.2 & 69.8 & \textbf{75.3} \\
& \textit{Visual-Language Avg. (6)} & 38.8 & 68.0 & 72.4 & \textbf{76.6} \\
\midrule
\multirow{3}{*}{\begin{tabular}[c]{@{}l@{}}Semantic Reasoning\\(object/receptacle)\end{tabular}}
& MultiCarrot-v1-test & 16.7 & 36.8 & 38.2 & \textbf{39.1} \\
& MultiPlate-v1-test & 11.8 & \textbf{29.4} & 28.3 & 21.9 \\
\cmidrule{2-6}
& \textit{Avg.} & 14.3 & 33.1 & \textbf{33.3} & 30.5 \\
\midrule
\multirow{3}{*}{Action Execution}
& PositionChangeTo-v1-test & 13.5 & 36.2 & \textbf{54.7} & 28.1 \\
& Position-v1-test & 31.2 & 54.5 & 55.0 & \textbf{56.3} \\
\cmidrule{2-6}
& \textit{Avg.} & 22.4 & 45.4 & \textbf{54.9} & 42.2 \\
\bottomrule
\end{tabular}
}
\label{tab:maniskill}
\end{minipage}
\vspace{-11pt}
\end{table}

\textbf{Online Improvement from Shared Few-shot SFT Initializations.}
Starting from the same few-shot SFT initialization for all compared online RL methods within each backbone, $\pi$-StepNFT raises average success from 57.6\% to 92.0\% for $\pi_0$ and from 77.1\% to 94.0\% for $\pi_{0.5}$, with means 2.0 and 2.5 percentage points higher than Flow-SDE + GRPO, respectively (Table~\ref{tab:libero}).
Incorporating critic-derived temporal-credit signals into the same ranking objective further improves the averages to 93.9\% and 96.0\%.
The largest gains from critic supervision occur on Long, where success increases from 86.7\% to 90.0\% for $\pi_0$ and from 79.8\% to 88.4\% for $\pi_{0.5}$, narrowing the corresponding gaps to PPO to 0.2 and 4.6 percentage points.
The ranking interface supports both binary outcomes and critic-derived temporal credit, while PPO retains the highest overall means.

\textbf{Higher Reported Mean Performance under Visual Shifts in ManiSkill.}
For $\pi_{0.5}$, across the five test-only visual perturbations, $\pi$-StepNFT obtains a mean success rate of 75.3\%, compared with 66.2\% for Flow-SDE and 69.8\% for Flow-Noise, and achieves the highest reported result on four of the five variants (Table~\ref{tab:maniskill}).
Relative to its in-distribution performance, mean success decreases by 9.0 percentage points under these perturbations, compared with drops of 24.7 and 19.9 percentage points for Flow-SDE and Flow-Noise, respectively.
This robustness advantage is concentrated in the tested visual appearance shifts.
The PPO baselines remain stronger on the in-distribution setting and achieve higher category averages under semantic and execution shifts.

\subsection{Mechanistic Analysis: Step-wise Guidance and Outcome-based Ranking}
\label{sec:diagnostic}

Figure~\ref{fig:ablation_531_532} traces the design choices from Original Diffusion-NFT to $\pi$-StepNFT: rollout weighting and noise treatment, transition targets, step selection, and the ranking objective. The final panels examine feedback choice on short- and long-horizon tasks.

\begin{figure}[t]
  \centering
  \begin{subfigure}[t]{0.31\textwidth}
  \centering
  \includegraphics[height=78.35pt,keepaspectratio]{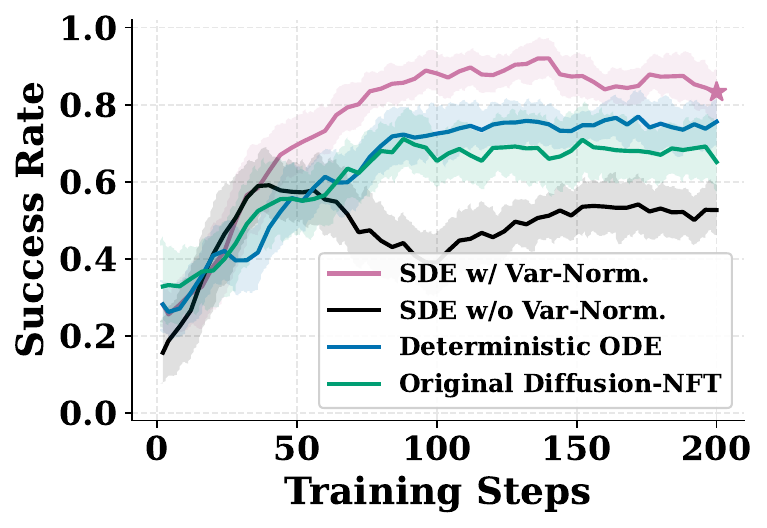}
  \vspace{-7pt}
  \caption{Rollout and noise treatment.}
  \vspace{-5pt}
  \label{fig:ablation_sde}
  \end{subfigure}
  \hfill
  \begin{subfigure}[t]{0.31\textwidth}
  \centering
  \includegraphics[height=78.35pt,keepaspectratio]{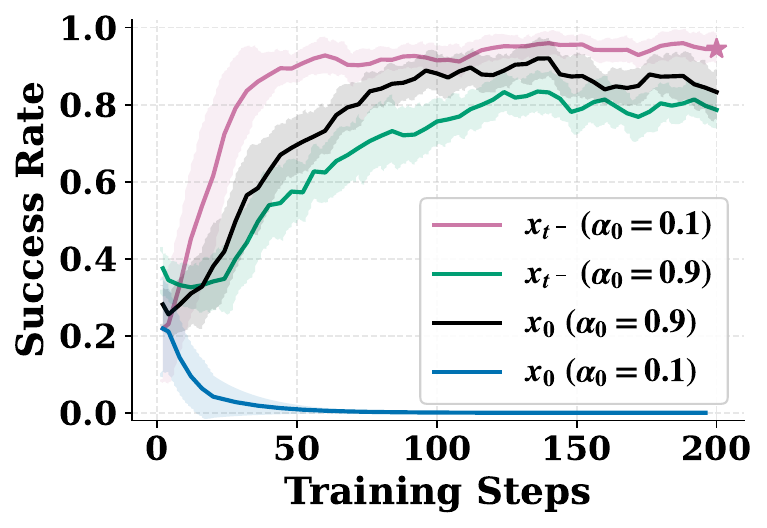}
  \vspace{-7pt}
  \caption{Regression target.}
  \vspace{-5pt}
  \label{fig:ablation_target}
  \end{subfigure}
  \hfill
  \begin{subfigure}[t]{0.31\textwidth}
  \centering
  \includegraphics[height=78.35pt,keepaspectratio]{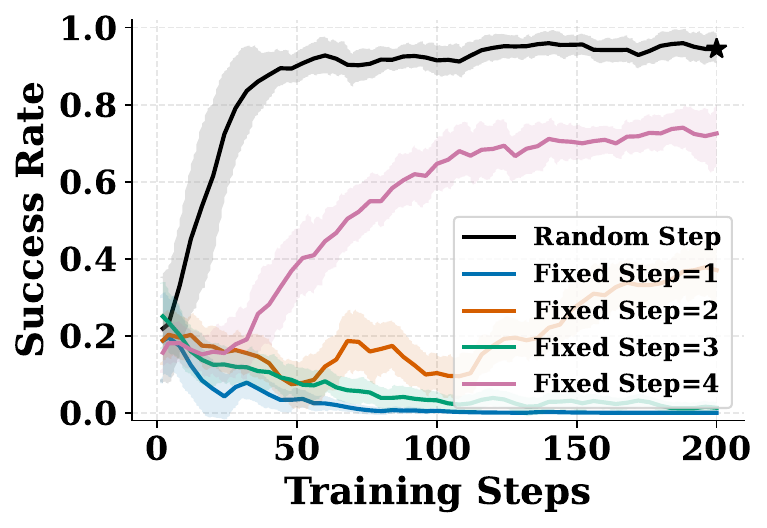}
  \vspace{-7pt}
  \caption{Step selection strategy.}
  \vspace{-5pt}
  \label{fig:ablation_step}
  \end{subfigure}
  \\[3pt]
  \begin{subfigure}[t]{0.31\textwidth}
  \centering
  \includegraphics[height=78.35pt,keepaspectratio]{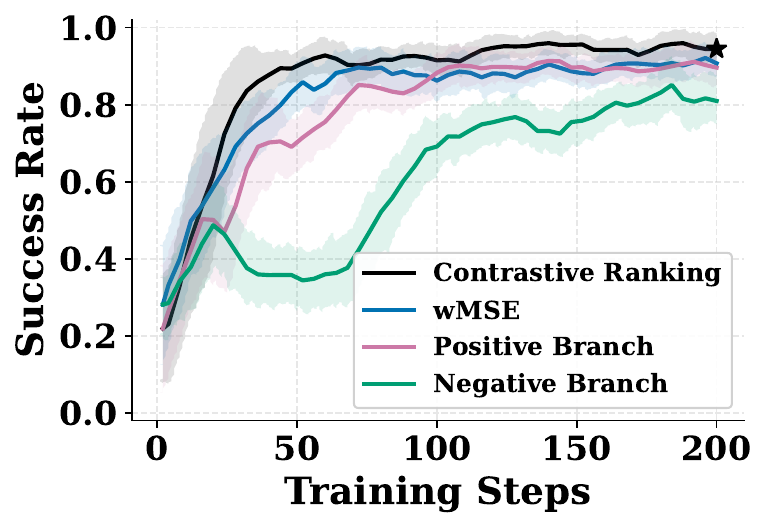}
  \vspace{-7pt}
  \caption{Learning objective.}
  \vspace{-5pt}
  \label{fig:ablation_loss}
  \end{subfigure}
  \hfill
  \begin{subfigure}[t]{0.31\textwidth}
  \centering
  \includegraphics[height=78.35pt,keepaspectratio]{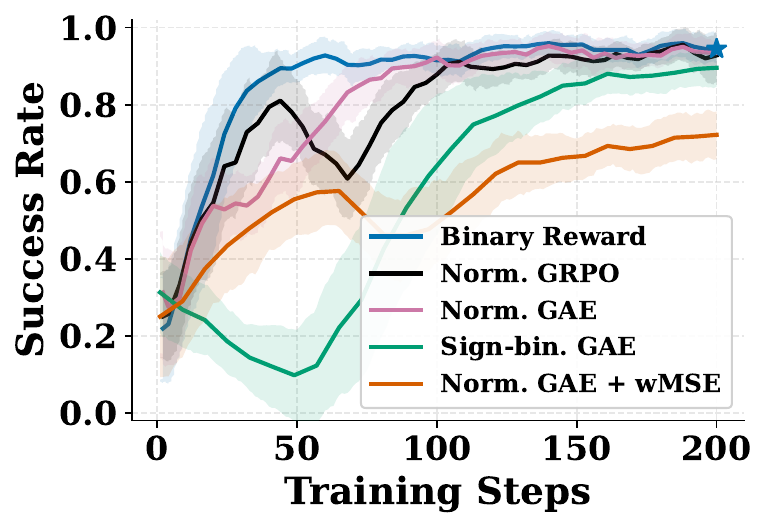}
  \vspace{-7pt}
  \caption{Preference signal.}
  \vspace{-5pt}
  \label{fig:ablation_credit}
  \end{subfigure}
  \hfill
  \begin{subfigure}[t]{0.31\textwidth}
  \centering
  \includegraphics[height=78.35pt,keepaspectratio]{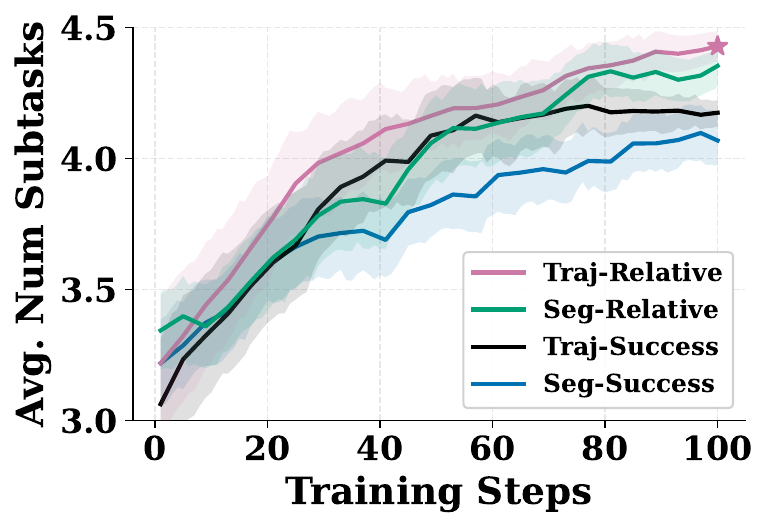}
  \vspace{-7pt}
  \caption{Subgoal feedback.}
  \vspace{-5pt}
  \label{fig:ablation_calvin}
  \end{subfigure}

  \caption{\textbf{Component-wise ablations of $\pi$-StepNFT.}
  (a) With $x_0$ targets, Original Diffusion-NFT uses ODE rollouts and GRPO-weighted MSE; Deterministic ODE uses binary-outcome weighting instead. SDE w/o and w/ Var-Norm then introduce stochastic rollouts without and with noise-aware normalization.
  (b) Immediate $x_{t^-}$ targets stabilize stochastic training.
  (c) Random step selection improves stability over fixed steps.
  (d) Ranking outperforms weighted-MSE and the separate Positive/Negative Branch ablations.
  (e) Binary outcomes perform best on the short-horizon ablation; ranking also supports soft scores.
  (f) Relative subgoal feedback helps in a controlled long-horizon diagnostic. Curves show five-seed means with $\pm1$ standard deviation.}
  \vspace{-20pt}
  \label{fig:ablation_531_532}
\end{figure}

\textbf{Wider space needs finer steps.}
Figure~\ref{fig:ablation_sde} holds the final $x_0$ target fixed under weighted regression and a conservative EMA schedule ($\alpha$: $0.9\to0.995$).
Original Diffusion-NFT is a direct transfer baseline that combines ODE rollouts with soft GRPO weighting and stays below 0.7 success.
Deterministic ODE keeps the same endpoint-derived regression but replaces soft GRPO scores with binary outcomes, improving performance before plateauing~\ding{182}.
Introducing SDE rollouts without variance normalization gives limited, unstable gains~\ding{183}; adding noise-aware normalization improves training~\ding{184}.
Figure~\ref{fig:ablation_target} then compares targets under stochastic rollouts: $x_0$ regression is only partly stabilized by conservative EMA ($\alpha_0=0.9$) and collapses at $\alpha_0=0.1$.
While final-output regression trains on intermediate states, it does not directly fit the observed adjacent SDE transition.
Targeting the immediate $x_{t^-}$ instead provides local, noise-aware supervision and improves stability and convergence~\ding{185}.
Its stability at $\alpha_0=0.1$, where $x_0$ regression collapses, shows that the gain is not solely due to conservative EMA updates.
Together, these comparisons support pairing wider SDE exploration with noise-aware supervision of the observed adjacent transition, rather than relying on exploration alone.

\begin{samepage}
\textbf{Random step selection improves stability.}
Figure~\ref{fig:ablation_step} shows non-monotonic fixed-step performance, with some choices collapsing. Random selection covers different denoising stages and achieves the best, most stable performance with single-transition updates.
\par
\end{samepage}

\textbf{Contrastive ranking beats weighted regression.}
Under the same step-wise transition formulation, Figure~\ref{fig:ablation_loss} compares objectives. The Positive Branch and Negative Branch ablations use all samples but retain only the respective branch loss. Both remain below weighted-MSE: positive-only improves slowly, while negative-only drops mid-training.
Weighted-MSE converges more slowly and reaches lower final success than ranking.
With binary feedback, weighted-MSE explicitly fits an outcome-selected absolute error; ranking compares both mirrored errors on each transition. This ordering favors outcome-selected regression over isolated-branch objectives, and relative comparison over weighted regression (Section~\ref{sec:weighted_regression_comparison}).

\textbf{Binary outcomes are strongest on the short-horizon ablation.}
Figure~\ref{fig:ablation_credit} next examines feedback: terminal binary outcomes, normalized GRPO and GAE advantages, and sign-binarized GAE on the default LIBERO-Object ablation task.
Binary outcomes converge fastest and achieve the highest final success.
Sign-binarized GAE remains weaker, showing that binarization alone does not explain the difference.
With the same normalized scores $r$ from PPO-critic GAE, ranking outperforms weighted-MSE. Holding this feedback fixed isolates the actor objective from the feedback source. The gain therefore extends beyond binary feedback.
Together with the main LIBERO results, this favors binary outcomes on the short-horizon Object task, while critic-derived signals yield their largest gains on Long.

\textbf{Relative subgoal feedback helps in a controlled long-horizon diagnostic.}
\label{sec:calvin_diagnostic}
For long-horizon feedback, Figure~\ref{fig:ablation_calvin} uses a fixed five-subgoal CALVIN sequence, which provides subgoal rewards unavailable in LIBERO-Long, to compare four label constructions from subtask completion.
Traj-Success broadcasts final success to all segments, whereas Seg-Success labels each segment by immediate subgoal completion.
Traj-Relative compares each rollout's completion count with the mean across the current rollout batch, while Seg-Relative applies the same relative comparison to each segment's subsequent completion count.
Direct segment labels underperform trajectory-success labels, suggesting that isolated subgoal completion need not support subsequent subtasks.
Among the relative variants, Traj-Relative performs best, while Seg-Relative is close but less stable, supporting relative long-horizon preferences in this diagnostic.

\begin{figure}[t]
  \centering
  \begin{subfigure}[t]{0.24\linewidth}
    \centering
    \includegraphics[width=\linewidth]{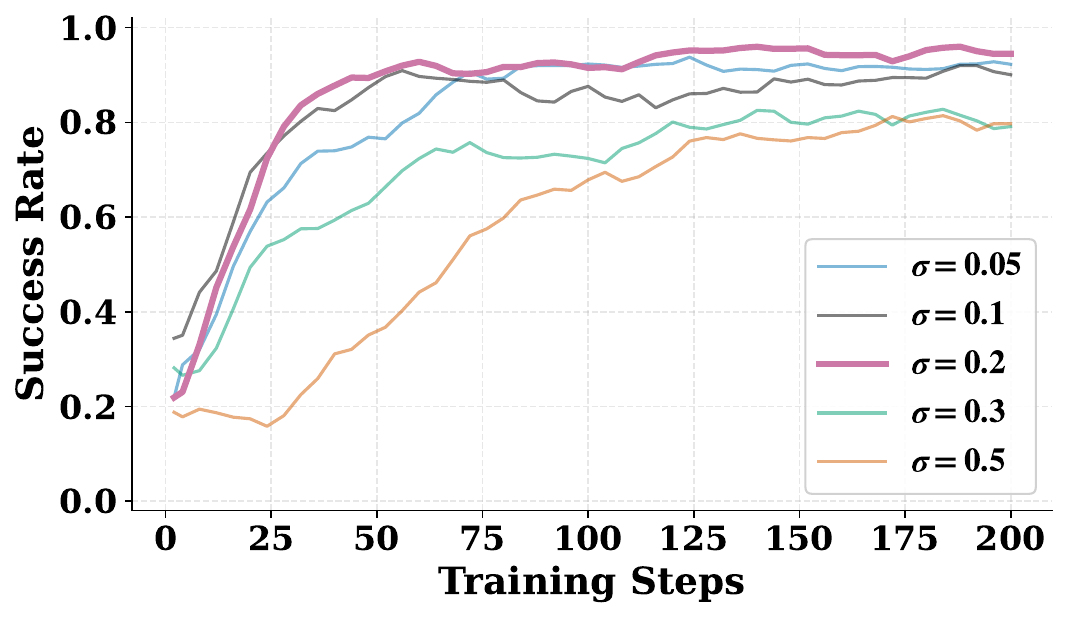}
    \vspace{-15pt}
    \caption{Noise level $\sigma$.}
    \vspace{-5pt}
    \label{fig:hparam_sigma}
  \end{subfigure}
  \hfill
  \begin{subfigure}[t]{0.24\linewidth}
    \centering
    \includegraphics[width=\linewidth]{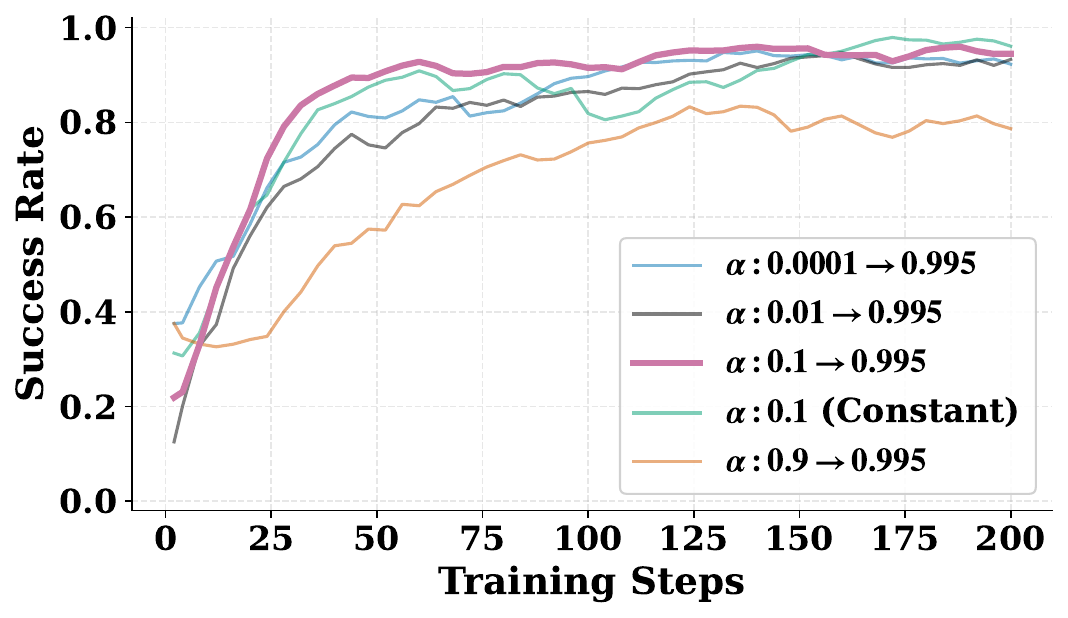}
    \vspace{-15pt}
    \caption{Decay $\alpha$.}
    \vspace{-5pt}
    \label{fig:hparam_alpha}
  \end{subfigure}
  \hfill
  \begin{subfigure}[t]{0.24\linewidth}
    \centering
    \includegraphics[width=\linewidth]{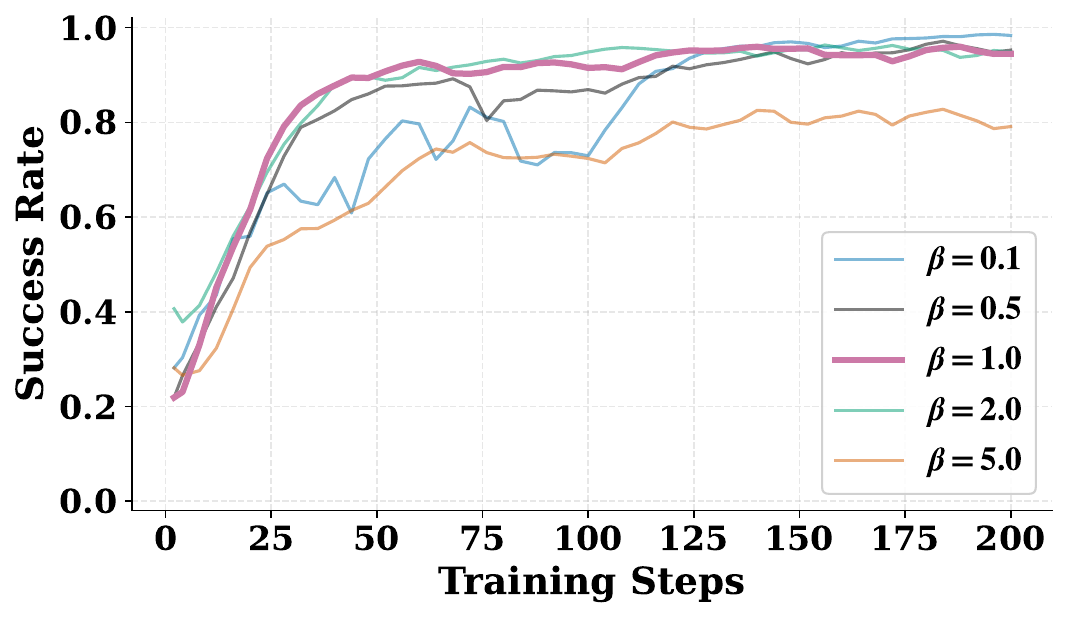}
    \vspace{-15pt}
    \caption{Branch scale $\beta$.}
    \vspace{-5pt}
    \label{fig:hparam_beta}
  \end{subfigure}
  \hfill
  \begin{subfigure}[t]{0.24\linewidth}
    \centering
    \includegraphics[width=\linewidth]{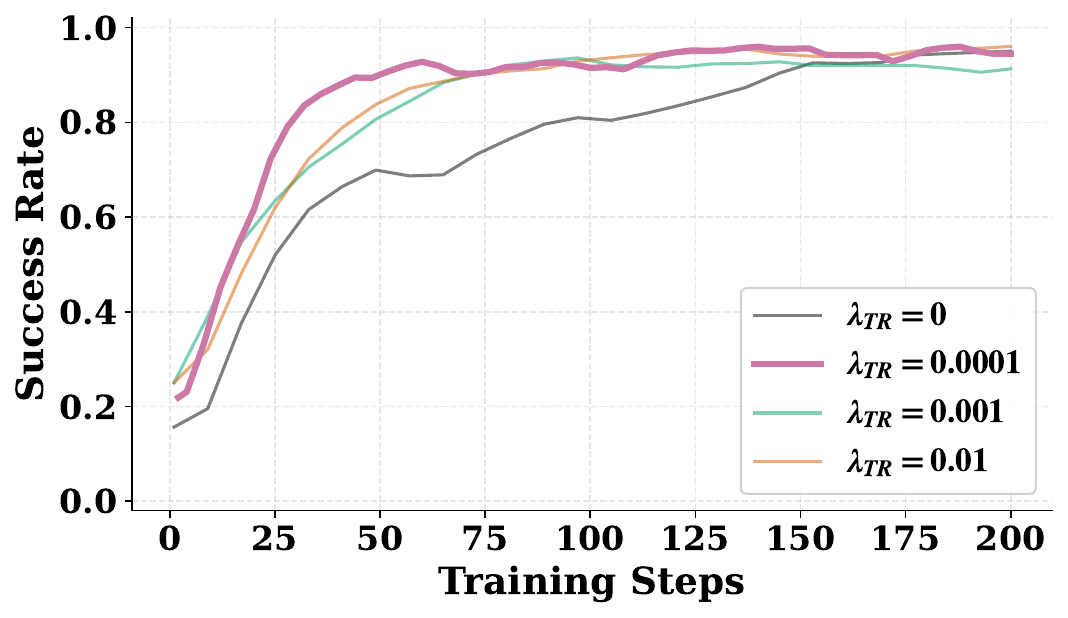}
    \vspace{-15pt}
    \caption{Trust reg. weight $\lambda_{\text{TR}}$.}
    \vspace{-5pt}
    \label{fig:hparam_lambda}
  \end{subfigure}

  \caption{\textbf{Hyperparameter sensitivity analysis.} Curves show five-seed means with $\pm1$ standard deviation; \textbf{bold pink} marks the main-experiment configuration.}
  \vspace{-18pt}
  \label{fig:hparam_sensitivity}
\end{figure}

\subsection{Hyperparameter Sensitivity Analysis}
\label{sec:sensitivity}
Figure~\ref{fig:hparam_sensitivity} illustrates the main hyperparameter trade-offs. For noise level $\sigma$, excessive noise widens the local transition distribution and impedes convergence, while insufficient noise limits local stochastic exploration. For EMA decay $\alpha$, high decay causes substantial rollout-policy lag, whereas constant or overly aggressive updates risk collapse. A dynamic schedule balances early improvement with final stability, consistent with the displacement-dependent correction in Appendix~\ref{app:finite_update_alignment}. For branch displacement scale $\beta$, values in $[1,2]$ perform best; at a fixed actor--rollout gap, increasing $\beta$ scales up normalized branch displacement and its finite-update correction, while smaller values weaken the ranking signal. For trust-region weight $\lambda_{\text{TR}}$, removing regularization reduces performance, whereas overly large values slow convergence; we therefore use $\lambda_{\text{TR}}=10^{-4}$.

\vspace{-4pt}
\section{Related Works}
\subsection{Online RL for VLAs}
VLA policies span action tokenization~\citep{rt2,openvla}, continuous regression~\citep{openvlaoft}, and flow-based action generation~\citep{octo,pi0,pi0.5,gr00t}. Tokenized policies expose categorical probabilities and continuous policies often use Gaussian heads~\citep{vlarl,rl4vla,simplevla}, whereas flow-based VLAs use multi-step ODE sampling whose final-action likelihoods are not directly tractable. Existing methods bypass this bottleneck through latent-space value learning~\citep{grrl,dsrl}, advantage-conditioned action generation~\citep{pi0.6}, and Q-learning-based actor distillation~\citep{fql}, or approximate likelihood-based optimization via denoising Gaussian parameterization~\citep{pirl} or conditional-flow-matching importance surrogates~\citep{fpo_vla}. In contrast, $\pi$-StepNFT avoids computing final-action likelihoods and ranks variance-normalized one-step SDE transition errors across negative-aware mirrored branches.

\subsection{Policy Optimization for Generative Models}
Policy optimization for diffusion and flow models treats denoising as a sequential decision process~\citep{ddpo,dpok,flowgrpo,reinflow}, requiring tractable denoising likelihoods or sampler backpropagation. SAC Flow stabilizes sampler backpropagation through velocity-reparameterized off-policy SAC~\citep{sacflow}. Other methods avoid sampler gradients through value-based diffusion-policy objectives~\citep{sdac}, adjoint-matched Q-learning~\citep{qam}, reward-weighted flow-matching surrogates~\citep{rwfm,ORW-CFM-W2,fpo}, or preference-ranking losses over final outputs or intermediate samples~\citep{diffdpo,lpo,tailorpo}. $\pi$-StepNFT instead ranks mirrored local transition models on the same observed SDE transition. Diffusion-NFT~\citep{diffusionnft} instead constructs likelihood-free mirrored branches in the forward process and uses regression targets derived from the final denoised output. $\pi$-StepNFT shares its negative-aware motivation but replaces endpoint-derived regression with observed-transition supervision.

\section{Conclusion}
\looseness=-1
We introduced $\pi$-StepNFT to address the exploration--supervision mismatch in online RL for flow-based VLAs. It maps binary or critic-derived feedback to local velocity updates by ranking variance-normalized errors on observed SDE transitions, without computing final-action likelihoods. Our analysis characterizes the outcome-conditioned gradient at the rollout policy and its finite-update correction under binary feedback. Experiments show gains over few-shot SFT on LIBERO and the highest reported mean success among evaluated online RL baselines across five test-only ManiSkill visual shifts.

\subsection*{AI use statement}
OpenAI Codex was used to assist with language editing, presentation and organization, literature search, \LaTeX{} formatting, and revisions to the description and interpretation of results. It was not used to generate experimental data.

\subsection*{Reproducibility statement}
Section~\ref{sec:definition}--\ref{sec:objective} and Algorithm~\ref{alg:stepnft} specify the transition construction and optimization objective, with derivations and proofs in Appendix~\ref{app:proofs}. Section~\ref{sec:experimental_setup} describes the evaluation and training protocol. Appendix~\ref{app:experiment_details} provides benchmark definitions, SFT initializations, rollout and update budgets, compute, hyperparameters, and ablation protocols; Appendix~\ref{app:additional_results} reports additional analyses, system profiling, and full ManiSkill OOD results.

\nocite{coefficients, iterative, step, discrete, ranking, towards, direct, smolvla, policy, deep, fine, denoising, score, continuous, 4090, H100}

\bibliography{paper}
\bibliographystyle{iclr2027_conference}

\newpage
\appendix
\onecolumn

\section{Scope of Compared Methods}
\label{app:comparison_scope}

\subsection{Image-generation Alignment versus Embodied Flow-policy RL}
\label{app:image_vs_embodied_flow_rl}

Online RL methods for flow or diffusion models have recently appeared in two related but distinct settings. The first is image-generation alignment, where the model produces a static sample and reward is evaluated on the final generated output. The second is embodied control, where generated actions are executed in a closed-loop MDP and small action errors can affect future states. Therefore, methods developed for image generation are conceptually relevant but not directly comparable as embodied-control baselines unless they are instantiated with a VLA policy, robotic environment rollouts, and matched interaction budgets.

\begin{table}[h]
\centering
\small
\caption{Relationship between online RL routes in image generation and their embodied flow-based VLA instantiations.}
\label{tab:method_mapping}
\resizebox{0.92\textwidth}{!}{%
\begin{tabular}{p{4.2cm} p{2.6cm} p{3.2cm} p{4.5cm}}
\toprule
\textbf{Conceptual route} &
\textbf{Method} &
\textbf{Embodied instantiation} &
\textbf{Compared in this paper} \\
\midrule
SDE policy-gradient route &
Flow-GRPO~\citep{flowgrpo} &
Flow-SDE + GRPO &
Yes; reported as $\pi_\texttt{RL}$ (Flow-SDE + GRPO) \\
Flow-noise policy-gradient route &
ReinFlow~\citep{reinflow} &
Flow-Noise + PPO &
Yes; reported as $\pi_\texttt{RL}$ (Flow-Noise + PPO) or Flow-Noise baseline \\
NFT likelihood-free route &
Diffusion-NFT~\citep{diffusionnft} &
Direct NFT transfer;\newline $\pi$-StepNFT &
Yes; direct NFT transfer is included as a diagnostic baseline, and $\pi$-StepNFT is our proposed repair \\
\bottomrule
\end{tabular}
}
\end{table}

\subsection{Comparison Criteria for Online RL Methods for VLAs}
\label{app:online_rl_comparison_criteria}

For controlled primary numerical comparisons, we prioritize methods that satisfy the following criteria: (i) same policy class: they use a flow-based VLA backbone or can be instantiated on the same $\pi_0/\pi_{0.5}$ backbone; (ii) same online interaction budget: they can be trained under a matched online interaction budget, including rollout trajectories, environment steps, parallel environments, and denoising steps; (iii) same supervision source: they do not require extra supervision beyond the online rollout feedback used in our setting, such as additional demonstrations, human interventions, or large heterogeneous robot data; and (iv) open implementation: their implementations are publicly available or sufficiently specified to avoid uncontrolled reimplementation choices. Table~\ref{tab:online_rl_vla_comparability} summarizes these criteria for recent online RL methods for VLAs and flow-based policies.

\begin{table}[h]
\centering
\small
\caption{Comparability of recent online RL methods for VLAs and flow-based policies. Primary comparisons prioritize a matched flow-based VLA policy class, matched online rollout budget, no extra supervision or heterogeneous training data, and public or sufficiently specified implementation details.}
\label{tab:online_rl_vla_comparability}
\setlength{\tabcolsep}{3pt}
\renewcommand{\arraystretch}{1.12}
\resizebox{0.90\textwidth}{!}{%
\begin{tabular}{p{3.0cm} p{3.6cm} c c c c c}
\toprule
\textbf{Method} &
\textbf{Key mechanism} &
\textbf{Flow Backbone} &
\textbf{Reproducible Recipe} &
\textbf{Matched Budget} &
\textbf{Online Feedback} &
\textbf{Primary Comparison} \\
\midrule
$\pi_\texttt{RL}$ Flow-SDE + GRPO &
\begin{tabular}[t]{@{}p{3.5cm}@{}}Flow-SDE +\\ GRPO surrogate\end{tabular} &
\cmark & \cmark & \cmark & \cmark & \cmark \\
$\pi_\texttt{RL}$ Flow-SDE + PPO &
\begin{tabular}[t]{@{}p{3.5cm}@{}}Flow-SDE +\\ critic-based PPO\end{tabular} &
\cmark & \cmark & \cmark & \cmark & \cmark \\
$\pi_\texttt{RL}$ Flow-Noise + PPO &
\begin{tabular}[t]{@{}p{3.5cm}@{}}Flow-Noise +\\ PPO likelihood\end{tabular} &
\cmark & \cmark & \cmark & \cmark & \cmark \\
FPO-VLA~\citep{fpo_vla} &
\begin{tabular}[t]{@{}p{3.5cm}@{}}CFM importance surrogate +\\ Q-ensemble\end{tabular} &
\cmark & \tmark & \tmark & \cmark & \xmark \\
FQL~\citep{fql} &
\begin{tabular}[t]{@{}p{3.5cm}@{}}Q-learning +\\ actor distillation\end{tabular} &
\tmark & \cmark & \tmark & \cmark & \xmark \\
DSRL~\citep{dsrl} &
Latent-noise steering policy &
\tmark & \cmark & \tmark & \cmark & \xmark \\
GR-RL~\citep{grrl} &
\begin{tabular}[t]{@{}p{3.5cm}@{}}Progress value +\\ latent noise prediction\end{tabular} &
\cmark / \tmark & \tmark & \tmark & \xmark & \xmark \\
RECAP / $\pi^{*}_{0.6}$~\citep{pi0.6} &
Advantage-conditioned VLA &
\cmark & \xmark / \tmark & \xmark & \xmark & \xmark \\
VLA-RL~\citep{vlarl} &
\begin{tabular}[t]{@{}p{3.5cm}@{}}Autoregressive VLA +\\ process reward\end{tabular} &
\xmark & \cmark & \tmark & \xmark & \xmark \\
RL4VLA~\citep{rl4vla} &
\begin{tabular}[t]{@{}p{3.5cm}@{}}Explicit-action VLA +\\ PPO-style RL\end{tabular} &
\xmark & \cmark & \tmark & \cmark & \xmark \\
SimpleVLA-RL~\citep{simplevla} &
\begin{tabular}[t]{@{}p{3.5cm}@{}}OpenVLA-OFT +\\ scalable RL system\end{tabular} &
\xmark & \cmark & \tmark & \cmark & \xmark \\
\bottomrule
\end{tabular}
}
\end{table}

Here, \cmark{} means that the criterion is clearly satisfied under the same experimental setting, \xmark{} means that it is clearly violated, and \tmark{} means that the method is relevant but the criterion cannot be certified as satisfied for a controlled comparison. This usually occurs when the method is evaluated with a different backbone or policy interface, requires nontrivial additional modules such as critics, Q-ensembles, latent controllers, actor distillation, or reward models, uses a different rollout/update protocol, or lacks public implementation details needed to match these factors.

\paragraph{Scope of compared methods.}
Our primary comparisons focus on methods that follow the standard two-stage online RL loop: first collecting rollouts with the current policy, and then optimizing the same flow-based VLA policy using the collected online feedback. This rollout--optimization structure makes it possible to control the main experimental budget, including the number of rollout trajectories, environment steps, parallel environments, and denoising steps. Under this criterion, the most controlled baselines are the $\pi_\texttt{RL}$ Flow-SDE and Flow-Noise variants~\citep{pirl}: they can be instantiated on the same $\pi_0/\pi_{0.5}$ backbone, trained with the same SFT initialization, and evaluated under the same simulator and rollout protocol. In particular, Flow-SDE + GRPO corresponds to the Flow-GRPO-style SDE policy-optimization route, while Flow-Noise + PPO corresponds to the ReinFlow-style noise-based policy-gradient route. These baselines therefore allow us to isolate whether our proposed method provides benefits under a matched online interaction budget.

\paragraph{Why not all related methods are primary baselines.}
Several recent online RL methods for VLAs or flow-based policies are highly relevant, but they do not provide clean primary baselines under the controlled rollout--optimization setting above. 
\begin{itemize}[leftmargin=1.2em]
    \item RECAP and $\pi^{*}_{0.6}$ train VLAs from heterogeneous experience, including demonstrations, on-policy collection, and expert teleoperated interventions, and they start from offline RL pretraining before downstream on-robot specialization. Their performance therefore depends on additional data sources and correction channels beyond simulator-only online rollout feedback. Moreover, without open-source code, its full pipeline cannot be reproduced under our matched budget. 
    \item FQL addresses flow-policy RL by training an expressive one-step policy with Q-learning rather than directly optimizing the iterative flow policy itself. This changes the trained policy object and introduces actor distillation and one-step policy extraction. 
    \item DSRL is also not a direct objective-level baseline. It adapts a pretrained diffusion or flow policy by training a small latent-noise policy, without modifying the base policy weights. Its pipeline treats noise as the RL action and can include warmup rollouts or offline-data usage through noise-aliasing value distillation, so its training object and rollout accounting differ from ours. 
    \item FPO-VLA is algorithmically closer because it targets online fine-tuning of flow-matching VLAs through a conditional-flow-matching importance surrogate. However, it also uses structure-aware credit assignment, clipped surrogate objectives, multi-step latent exploration, and Q-ensemble value estimation. It also lacks public code and sufficient hyperparameter details for reliable reproduction under the same backbone, solver schedule, rollout budget, and update protocol. 
\end{itemize}

For these reasons, we discuss RECAP, $\pi^{*}_{0.6}$, FQL, DSRL, and FPO-VLA as important related routes, while using the matched $\pi_\texttt{RL}$ Flow-SDE and Flow-Noise variants as the primary numerical baselines.

\newpage
\section{Theoretical Analysis and Proofs}
\label{app:proofs}
This section provides the detailed derivations used by Proposition~\ref{prop:local_alignment} and Lemma~\ref{lem:wmse_ranking_decomposition}.

\subsection{Notation Summary}
\label{app:notation_summary}

Table~\ref{tab:notation_summary} summarizes the main symbols used in the method description and theoretical analysis.

\begin{table}[h]
\centering
\caption{Terminology and notation summary.}
\label{tab:notation_summary}
\small
\setlength{\tabcolsep}{4pt}
\renewcommand{\arraystretch}{1.08}
\resizebox{1.0\textwidth}{!}{%
\begin{tabular}{p{8.5cm} p{5.5cm}}
\toprule
\textbf{Term / Merged Expressions} & \textbf{Notation} \\
\midrule
Environment step index / episode horizon & $i$, $H$ \\
Denoising time / solver schedule / solver steps & $t$, $t_j$, $K$ \\
Environment trajectory / episode trajectory & $\tau$, $s_i$, $a_i$ \\

Final denoised output & $x_0$ \\
Gaussian prior / initial noise & $x_1$ \\
Target velocity & $u_t$ \\

SDE transition state / step-wise transition & $x_t$, $x_{t_j}$, $x_{t^-}$, $\delta_t$, $x_t \to x_{t^-}$ \\
Stochastic noise terms & $\sigma_t$, $\epsilon$ \\
Covariance / variance-normalization term & $\Sigma_t$ \\
Affine transition mean / affine coefficients & $\mu_t$, $U_t$, $B_t$ \\

Positive / negative rollout sets & $\mathcal{D}^{+}$, $\mathcal{D}^{-}$ \\
Positive / negative latent policies & $\pi^{+}$, $\pi^{-}$ \\
Latent outcome / oracle outcome & $o$ \\
Optimality score & $r$ \\
Signed preference label & $y$ \\

Flow policy / rollout policy / old policy & $\pi_\theta$, $\pi^{\mathrm{old}}$ \\
Velocity field / old velocity field & $v_\theta$, $v^{\mathrm{old}}$ \\
Positive / negative mirrored velocity branches & $v_\theta^{+}$, $v_\theta^{-}$ \\

Time-indexed velocity values / time-indexed mirrored branches & $v_{\theta,t}$, $v_t^{\mathrm{old}}$, $v_{\theta,t}^{+}$, $v_{\theta,t}^{-}$ \\
Local velocity deviation & $\Delta v_{\theta,t}$ \\
Variance-normalized transition errors / mirror errors & $E_{\theta,t}^{+}$, $E_{\theta,t}^{-}$ \\

Transition density / branch transition density & $q_{\theta,t}$, $q_{\theta,t}^{+}$, $q_{\theta,t}^{-}$ \\
Transition mean / branch transition means & $\mu_{\theta,t}$, $\mu_{\theta,t}^{+}$, $\mu_{\theta,t}^{-}$ \\
Rollout transition mean & $\mu_t^{\mathrm{old}}$ \\
Outcome-conditioned transition means & $\mu_t^{+}$, $\mu_t^{-}$ \\
Oracle improvement direction & $\Delta$ \\
\leavevmode Oracle outcome-conditioned transition-mean difference & $\Delta\mu_t^\star$ \\

Branch displacement scale & $\beta$ \\
Trust-region regularizer & $\lambda_{\mathrm{TR}}$ \\

Rollout residual & $e_t$ \\
Mean-space branch displacement & $d_t$ \\
Normalized branch-displacement magnitude & $\rho_t$ \\
Finite-update correction & $-\zeta_t$ \\

Outcome probability & $\alpha$ \\
Alignment coefficient & $\kappa_t$ \\
Score-mismatch bias & $\varepsilon_t^r$ \\

Conditional flow matching objective & $\mathcal{L}_{\text{CFM}}$ \\
Diffusion-NFT objective & $\mathcal{L}_{\mathrm{NFT}}$ \\
Contrastive ranking / logistic ranking loss & $\ell_t$ \\
Weighted regression / wMSE & $\ell_t^{\mathrm{wMSE}}$ \\
\bottomrule
\end{tabular}
}
\end{table}

\newpage
\subsection{Affine Mean of the SDE Transition}
\label{app:proof_Affine Mean Derivation}

We first recall the one-step ODE and SDE updates underlying Eq.~\eqref{eq:affine_mean}. With step size $\delta_t>0$ and $t^-=t-\delta_t$, deterministic ODE sampling uses the Euler update
\begin{equation}
    x_{t^-}
    =
    x_t
    -
    \delta_t\,v_\theta(x_t,t,c).
\end{equation}
The reverse-time SDE sampler uses the Euler--Maruyama update
\begin{equation}
    x_{t^-}
    =
    x_t
    -
    \delta_t
    \left[
    v_\theta(x_t,t,c)
    +
    \frac{\sigma_t^2}{2t}
    \Bigl(x_t+(1-t)v_\theta(x_t,t,c)\Bigr)
    \right]
    +
    \sigma_t\sqrt{\delta_t}\,\epsilon,
    \qquad
    \epsilon\sim\mathcal{N}(0,I).
\end{equation}
Hence the SDE transition is Gaussian, $q_{\theta,t}(x_{t^-}\mid x_t,c)=\mathcal{N}(\mu_{\theta,t}(x_t,c),\Sigma_t)$ with $\Sigma_t=\sigma_t^2\delta_t I$, and its mean is the deterministic part of the SDE update.

We now derive the affine mean in Eq.~\eqref{eq:affine_mean}. Fix $(x_t,t,c)$ and write $v=v_\theta(x_t,t,c)$. The flow-SDE solver writes the deterministic transition mean through the rectified-flow endpoint estimates
\begin{equation}
    x_0^{\mathrm{pred}}=x_t-tv,
    \qquad
    x_1^{\mathrm{pred}}=x_t+(1-t)v,
\end{equation}
with Euler-Maruyama weights
\begin{equation}
    w_0=1-t+\delta_t,
    \qquad
    w_1=t-\delta_t-\frac{\sigma_t^2\delta_t}{2t}.
\end{equation}
Thus
\begin{align}
    \mu_t(v)
    &=
    w_0 x_0^{\mathrm{pred}}+w_1 x_1^{\mathrm{pred}}
    \nonumber\\
    &=
    w_0(x_t-tv)+w_1(x_t+(1-t)v)
    \nonumber\\
    &=
    (w_0+w_1)x_t+\bigl(-tw_0+(1-t)w_1\bigr)v .
\end{align}
The coefficient of $x_t$ is
\begin{align}
    w_0+w_1
    &=
    (1-t+\delta_t)
    +
    \left(t-\delta_t-\frac{\sigma_t^2\delta_t}{2t}\right)
    =
    1-\frac{\sigma_t^2\delta_t}{2t},
\end{align}
and the coefficient of $v$ is
\begin{align}
    -tw_0+(1-t)w_1
    &=
    -t(1-t+\delta_t)
    +(1-t)\left(t-\delta_t-\frac{\sigma_t^2\delta_t}{2t}\right)
    \nonumber\\
    &=
    (-t+t^2-t\delta_t)
    +
    \left(t-t^2-\delta_t+t\delta_t-\frac{(1-t)\sigma_t^2\delta_t}{2t}\right)
    \nonumber\\
    &=
    -\delta_t-\frac{(1-t)\sigma_t^2\delta_t}{2t}.
\end{align}
Therefore
\begin{equation}
    \mu_t(v)=U_t(x_t,t)+B_t(t)v,
    \qquad
    U_t(x_t,t)=\left(1-\frac{\sigma_t^2\delta_t}{2t}\right)x_t,
    \qquad
    B_t(t)=-\delta_t-\frac{(1-t)\sigma_t^2\delta_t}{2t},
\end{equation}
which is the affine form used in Eq.~\eqref{eq:affine_mean}. \hfill$\square$

\subsection{Log-ratio Form of the Ranking Loss}
\label{app:proof_Log-Likelihood Ratio}

We prove the log-ratio identity in Eq.~\eqref{eq:step_ranking}. For the mirrored branches,
$q_{\theta,t}^{\pm}(x_{t^-}\mid x_t,c)=\mathcal N(\mu_{\theta,t}^{\pm},\Sigma_t)$ with shared covariance $\Sigma_t$. For a Gaussian transition $q(x)=\mathcal N(\mu,\Sigma)$,
\begin{equation}
    \log q(x)
    =
    -\frac12(x-\mu)^\top\Sigma^{-1}(x-\mu)
    -\frac12\log\det(2\pi\Sigma).
\end{equation}
Subtracting the two branch log-densities cancels the normalization constant:
\begin{align}
    \log
    \frac{q_{\theta,t}^{+}(x_{t^-}\mid x_t,c)}
    {q_{\theta,t}^{-}(x_{t^-}\mid x_t,c)}
    &=
    -\frac12
    \left[
    (x_{t^-}-\mu_{\theta,t}^{+})^\top
    \Sigma_t^{-1}
    (x_{t^-}-\mu_{\theta,t}^{+})
    -
    (x_{t^-}-\mu_{\theta,t}^{-})^\top
    \Sigma_t^{-1}
    (x_{t^-}-\mu_{\theta,t}^{-})
    \right]
    \nonumber\\
    &=
    -\frac12
    \left[
    \left\|x_{t^-}-\mu_{\theta,t}^{+}\right\|_{\Sigma_t^{-1}}^2
    -
    \left\|x_{t^-}-\mu_{\theta,t}^{-}\right\|_{\Sigma_t^{-1}}^2
    \right]
    \nonumber\\
    &=
    -\frac12
    \left(E_{\theta,t}^{+}-E_{\theta,t}^{-}\right).
\end{align}
The Gaussian normalizers cancel because the two transition distributions have the same covariance. Substituting this identity gives
\begin{equation}
    \ell_t(\theta)
    =
    \operatorname{softplus}
    \left(\frac12 y(E_{\theta,t}^{+}-E_{\theta,t}^{-})\right)
    =
    \operatorname{softplus}
    \left(
    -y
    \log
    \frac{q_{\theta,t}^{+}(x_{t^-}\mid x_t,c)}
    {q_{\theta,t}^{-}(x_{t^-}\mid x_t,c)}
    \right).
\end{equation}
\hfill$\square$

\subsection{Proof of Proposition~\ref{prop:local_alignment}}
\label{app:proof_Gradient Form and Alignment}

Fix $(x_t,t,c)$ and the rollout distribution $x_{t^-}\mid x_t,t,c\sim\mathcal N(\mu_t^{\mathrm{old}},\Sigma_t)$, with $\Sigma_t\succ0$. Conditioning below suppresses $t$. Cached states, old velocities, transition coefficients, and scores are held fixed with respect to $\theta$; all required conditional expectations are assumed finite. Let $\theta_0$ satisfy $v_{\theta_0,t}=v_t^{\mathrm{old}}$, and define $J_0=\left.\partial v_{\theta,t}/\partial\theta\right|_{\theta=\theta_0}$.

\paragraph{Step 1: Constructed branch means.}
Let $\mu_t^{\mathrm{old}}=\mu_t(v_t^{\mathrm{old}})$ be the rollout transition mean, $e_t=x_{t^-}-\mu_t^{\mathrm{old}}$ be its residual, and $d_t=\mu_{\theta,t}^{+}-\mu_t^{\mathrm{old}}$ be the positive-branch displacement.
Using $v_{\theta,t}^{\pm}=v_t^{\mathrm{old}}\pm\beta\Delta v_{\theta,t}$ and the affine mean $\mu_t(v)=U_t+B_t v$ from Eq.~\eqref{eq:affine_mean},
\begin{equation}
    d_t
    =
    \beta B_t\Delta v_{\theta,t},
    \qquad
    \mu_{\theta,t}^{\pm}
    =
    \mu_t^{\mathrm{old}}\pm d_t .
    \label{eq:constructed_mu_pm}
\end{equation}

\paragraph{Step 2: Error difference.}
The variance-normalized transition errors satisfy
\begin{align}
    E_{\theta,t}^{+}-E_{\theta,t}^{-}
    &=
    \|e_t-d_t\|_{\Sigma_t^{-1}}^2
    -
    \|e_t+d_t\|_{\Sigma_t^{-1}}^2
    \nonumber\\
    &=
    -4\langle \Sigma_t^{-1}e_t,d_t\rangle .
    \label{eq:appendix_error_difference}
\end{align}

\paragraph{Step 3: Gradient form.}
Let
$z_t=\frac12 y(E_{\theta,t}^{+}-E_{\theta,t}^{-})$.
Since $\ell_t(\theta)=\operatorname{softplus}(z_t)$, write
$s(z_t)=1/(1+\exp(-z_t))$ for the sigmoid function, avoiding confusion with the SDE noise scale $\sigma_t$. From Eq.~\eqref{eq:appendix_error_difference},
\begin{equation}
    z_t=-2y\langle \Sigma_t^{-1}e_t,d_t\rangle .
\end{equation}
Because the rollout transition mean $\mu_t^{\mathrm{old}}$ and residual $e_t$ are fixed in the current batch, only
$d_t=\beta B_t\Delta v_{\theta,t}$ depends on $\theta$. Thus
\begin{equation}
    -\nabla_\theta\ell_t(\theta)
    =
    2\beta\,s(z_t)y
    \left(\frac{\partial v_{\theta,t}}{\partial\theta}\right)^\top
    B_t^\top\Sigma_t^{-1}e_t .
    \label{eq:appendix_gradient_form}
\end{equation}

For a $d$-dimensional velocity field and $p$ trainable parameters, $J_0\in\mathbb R^{d\times p}$ and $B_t,\Sigma_t\in\mathbb R^{d\times d}$; a scalar $B_t$ acts as a multiple of the identity. At $\theta=\theta_0$, $d_t=0$, so $z_t=0$ and $s(z_t)=1/2$ exactly. Hence
\begin{equation}
    \left.-\nabla_\theta\ell_t(\theta)\right|_{\theta=\theta_0}
    =\beta y J_0^\top B_t^\top\Sigma_t^{-1}e_t .
    \label{eq:appendix_gradient_compact}
\end{equation}

\paragraph{Step 4: Conditional expectation, outcome-conditioned transition-mean difference, and soft-score bias.}
Let $o\in\{0,1\}$ denote the latent outcome that defines the oracle positive and negative transition splits, and let $r\in[0,1]$ be the optimality score used in the ranking loss, with $y=2r-1$. Let
\begin{equation}
    \alpha(x_t,c)=\mathbb P(o=1\mid x_t,c),
    \qquad
    \kappa_t(x_t,c)=2\alpha(x_t,c)\bigl(1-\alpha(x_t,c)\bigr).
\end{equation}
At $\theta_0$, $J_0$ and the transition coefficients are fixed given $(x_t,t,c)$. Equation~\eqref{eq:appendix_gradient_compact} therefore reduces the expected gradient to the conditional moment $\mathbb E[y e_t\mid x_t,c]$.

Under the rollout SDE transition,
$\mu_t^{\mathrm{old}}=\mathbb E[x_{t^-}\mid x_t,c]$, so
\begin{equation}
    \mathbb E[e_t\mid x_t,c]=0.
    \label{eq:et_zero_mean}
\end{equation}
Moreover, with the outcome-conditioned transition means
$\mu_t^{+}(x_t,c)=\mathbb E[x_{t^-}\mid x_t,c,o=1]$ and
$\mu_t^{-}(x_t,c)=\mathbb E[x_{t^-}\mid x_t,c,o=0]$ from Section~\ref{sec:validity}, the law of total expectation gives
\begin{equation}
    \mu_t^{\mathrm{old}}
    =
    \alpha(x_t,c)\mu_t^{+}(x_t,c)
    +
    \bigl(1-\alpha(x_t,c)\bigr)\mu_t^{-}(x_t,c).
    \label{eq:old_mean_total_expectation}
\end{equation}

We first compute the signed residual induced by the latent oracle split. Using Eq.~\eqref{eq:et_zero_mean},
\begin{align}
    \mathbb E[(2o-1)e_t\mid x_t,c]
    &=
    2\,\mathbb E[o e_t\mid x_t,c]
    \nonumber\\
    &=
    2\alpha(x_t,c)\bigl(\mu_t^{+}(x_t,c)-\mu_t^{\mathrm{old}}\bigr)
    \nonumber\\
    &=
    2\alpha(x_t,c)\bigl(1-\alpha(x_t,c)\bigr)
    \Delta\mu_t^\star(x_t,c)
    \nonumber\\
    &=
    \kappa_t(x_t,c)\Delta\mu_t^\star(x_t,c),
    \label{eq:oracle_signed_residual}
\end{align}
where
$\Delta\mu_t^\star(x_t,c)=\mu_t^{+}(x_t,c)-\mu_t^{-}(x_t,c)$.
For non-degenerate states with $\alpha(x_t,c)\in(0,1)$, the scalar $\kappa_t(x_t,c)$ is positive; when $\alpha$ is $0$ or $1$, the conditional direction degenerates.

For a soft score $r\in[0,1]$, the signed label is $y=2r-1$. Decomposing it around the latent label gives
\begin{equation}
    2r-1
    =
    (2o-1)+2(r-o).
\end{equation}
Therefore,
\begin{align}
    \mathbb E[y e_t\mid x_t,c]
    &=
    \mathbb E[(2r-1)e_t\mid x_t,c]
    \nonumber\\
    &=
    \mathbb E[(2o-1)e_t\mid x_t,c]
    +
    2\mathbb E[(r-o)e_t\mid x_t,c]
    \nonumber\\
    &=
    \kappa_t(x_t,c)\Delta\mu_t^\star(x_t,c)
    +
    \varepsilon_t^r(x_t,c),
    \label{eq:soft_score_signed_residual}
\end{align}
where
\begin{equation}
    \varepsilon_t^r(x_t,c)
    =
    2\mathbb E[(r-o)e_t\mid x_t,c]
\end{equation}
is the score-mismatch bias. Thus, the oracle direction is still determined by the latent positive and negative transition means, while the soft score contributes an additional term measuring whether the score error $r-o$ is systematically correlated with the local transition residual $e_t$.

Taking conditional expectation of the exact gradient at $\theta_0$ in Eq.~\eqref{eq:appendix_gradient_compact} gives
\begin{equation}
    \left.\mathbb E[-\nabla_\theta\ell_t(\theta)\mid x_t,c]\right|_{\theta=\theta_0}
    =\beta J_0^\top B_t^\top\Sigma_t^{-1}\mathbb E[y e_t\mid x_t,c].
\end{equation}
Substituting Eq.~\eqref{eq:soft_score_signed_residual} yields
\begin{equation}
    \left.\mathbb E[-\nabla_\theta\ell_t(\theta)\mid x_t,c]\right|_{\theta=\theta_0}
    =\beta J_0^\top B_t^\top\Sigma_t^{-1}
    \left[
        \kappa_t(x_t,c)\Delta\mu_t^\star(x_t,c)
        +
        \varepsilon_t^r(x_t,c)
    \right],
    \label{eq:soft_score_gradient_final}
\end{equation}
which proves the soft-score relation in Proposition~\ref{prop:local_alignment}.

In the hard-label case $r=o$, we have $\varepsilon_t^r(x_t,c)=0$, and Eq.~\eqref{eq:soft_score_gradient_final} reduces to
\begin{equation}
    \left.\mathbb E[-\nabla_\theta\ell_t(\theta)\mid x_t,c]\right|_{\theta=\theta_0}
    =\beta\kappa_t J_0^\top B_t^\top\Sigma_t^{-1}\Delta\mu_t^\star(x_t,c).
\end{equation}
For $0<\alpha(x_t,c)<1$, $\beta\kappa_t>0$, giving the interpretation in terms of the outcome-conditioned transition-mean difference at the rollout policy.

For posterior-like soft scores $r(x_0,c)=\mathbb P(o=1\mid x_0,c)$, the bias term also vanishes when the intermediate transition provides no additional information about the latent positive outcome once $x_0$ and $c$ are known,
\begin{equation}
    \mathbb P(o=1\mid x_0,c,x_t,x_{t^-})
    =
    \mathbb P(o=1\mid x_0,c).
    \label{eq:transition_sufficiency_condition}
\end{equation}
Indeed, under Eq.~\eqref{eq:transition_sufficiency_condition},
\begin{align}
    \mathbb E[(r-o)e_t\mid x_t,c]
    &=
    \mathbb E\!\left[
        \mathbb E[(r-o)e_t\mid x_0,c,x_t,x_{t^-}]
        \mid x_t,c
    \right]
    \nonumber\\
    &=
    \mathbb E\!\left[
        e_t\bigl(r-\mathbb E[o\mid x_0,c,x_t,x_{t^-}]\bigr)
        \mid x_t,c
    \right]
    \nonumber\\
    &=
    \mathbb E\!\left[
        e_t\bigl(\mathbb P(o=1\mid x_0,c)-\mathbb P(o=1\mid x_0,c,x_t,x_{t^-})\bigr)
        \mid x_t,c
    \right]
    =
    0.
\end{align}
At the rollout policy, hard binary feedback yields the exact interpretation in terms of the outcome-conditioned transition-mean difference, and soft scores recover the same expected gradient when their mismatch from the latent split is uncorrelated with the local transition residual. This completes the proof of Proposition~\ref{prop:local_alignment}. \hfill$\square$

\paragraph{Remark: connection to Diffusion-NFT splits.}
Diffusion-NFT motivates positive and negative branches through distribution splits on final denoised output $x_0$. The proof above does not require explicitly training those splits: Proposition~\ref{prop:local_alignment} only uses the rollout transition, the mirrored branches, and the outcome-conditioned transition means $\mu_t^{+}$ and $\mu_t^{-}$. The oracle positive/negative split remains the latent variable $o\in\{0,1\}$; soft scores $r\in[0,1]$ act as approximate scores for this split and enter the expected gradient at the rollout policy through the score-mismatch bias $\varepsilon_t^r$.

\subsection{Finite-Update Analysis under Binary Feedback}
\label{app:finite_update_alignment}

Proposition~\ref{prop:local_alignment} gives the exact expected gradient at the rollout policy. Here we retain the full logistic coefficient to characterize finite displacement under binary feedback $r=o$, $y\in\{-1,+1\}$, using the same fixed rollout distribution and cached quantities. General soft scores require retaining score-dependent logistic weights as well as score mismatch.

\paragraph{Step 1: Exact finite-update gradient.}
Recall from Eq.~\eqref{eq:constructed_mu_pm} that
\begin{equation}
    e_t=x_{t^-}-\mu_t^{\mathrm{old}},
    \qquad
    d_t=\beta B_t\left(v_{\theta,t}-v_t^{\mathrm{old}}\right),
\end{equation}
and let
\begin{equation}
    a_t=\left\langle\Sigma_t^{-1}e_t,d_t\right\rangle.
\end{equation}
Lemma~\ref{lem:wmse_ranking_decomposition} then gives the exact binary-feedback loss
\begin{equation}
    \ell_t(\theta)=\operatorname{softplus}(-2y a_t).
\end{equation}
These relations do not require $d_t$ to be small. Using
\begin{equation}
    2y\,s(-2ya_t)=y-\tanh(a_t),
    \qquad y\in\{-1,+1\},
\end{equation}
where $s(\cdot)$ denotes the sigmoid function, the exact single-sample descent direction is
\begin{equation}
\begin{aligned}
    -\nabla_\theta\ell_t
    ={}&
    \beta
    \left(
        \frac{\partial v_{\theta,t}}{\partial\theta}
    \right)^{\!\top}
    B_t^\top\Sigma_t^{-1}
    \left[
        y e_t
        -
        e_t\tanh\!\left(
            \left\langle\Sigma_t^{-1}e_t,d_t\right\rangle
        \right)
    \right].
    \label{eq:finite_update_sample_direction}
\end{aligned}
\end{equation}
The first term is the oracle-aligned component from Proposition~\ref{prop:local_alignment}; the second is the finite-update correction induced by the sigmoid multiplier's dependence on the rollout residual.

\paragraph{Step 2: Conditional expectation of the correction.}
Equation~\eqref{eq:oracle_signed_residual} gives
\begin{equation}
    \mathbb E[y e_t\mid x_t,c]
    =
    \kappa_t\Delta\mu_t^\star.
    \label{eq:finite_update_oracle_term}
\end{equation}
Under the rollout SDE, the residual conditioned on $(x_t,c)$ satisfies
\begin{equation}
    e_t\mid x_t,c\sim\mathcal N(0,\Sigma_t).
\end{equation}
Applying the multivariate Stein identity
\begin{equation}
    \mathbb E[e_t f(e_t)]
    =
    \Sigma_t\mathbb E[\nabla_{e_t}f(e_t)]
\end{equation}
with
$f(e_t)=\tanh(e_t^\top\Sigma_t^{-1}d_t)$ yields
\begin{align}
    &\mathbb E\!\left[
        e_t\tanh\!\left(e_t^\top\Sigma_t^{-1}d_t\right)
        \middle|x_t,c
    \right]
    \nonumber\\
    &\qquad=
    d_t\,
    \mathbb E_{Z\sim\mathcal N(0,1)}\!\left[
        \operatorname{sech}^2\!\left(
            \left\|\Sigma_t^{-1/2}d_t\right\|_2 Z
        \right)
    \right],
    \label{eq:finite_update_stein}
\end{align}
where $Z$ is the standardized projection of the rollout residual onto the branch-displacement direction. Define
\begin{equation}
    \rho_t
    :=
    \left\|\Sigma_t^{-1/2}d_t\right\|_2,
    \qquad
    \zeta_t
    :=
    d_t\,
    \mathbb E_{Z\sim\mathcal N(0,1)}\!\left[
        \operatorname{sech}^2(\rho_t Z)
    \right].
    \label{eq:rho_zeta_definition}
\end{equation}
Here $\rho_t$ is the branch-mean displacement measured in units of the local transition standard deviation, and $\zeta_t$ points in the same direction as $d_t$. Substituting Eqs.~\eqref{eq:finite_update_oracle_term} and~\eqref{eq:finite_update_stein} into Eq.~\eqref{eq:finite_update_sample_direction} gives
\begin{equation}
    \mathbb E[-\nabla_\theta\ell_t\mid x_t,c]
    =
    \beta
    \left(
        \frac{\partial v_{\theta,t}}{\partial\theta}
    \right)^{\!\top}
    B_t^\top\Sigma_t^{-1}
    \left[
        \kappa_t\Delta\mu_t^\star-\zeta_t
    \right].
    \label{eq:finite_update_expected_direction}
\end{equation}
The finite-update correction $-\zeta_t$ opposes the current displacement in transition-mean space. This damping effect arises from logistic weighting and varies with $d_t$. It is distinct from both the fixed quadratic penalty in weighted regression and the separately added trust-region regularizer.

\paragraph{Step 3: Local limit and magnitude of the correction.}
At $\theta_0$, $d_t=0$ and $\zeta_t=0$, so Eq.~\eqref{eq:finite_update_expected_direction} reduces exactly to Proposition~\ref{prop:local_alignment} with $\varepsilon_t^r=0$. For fixed $\Sigma_t$, the correction vanishes as $\rho_t\to0$, since $\mathbb E_Z[\operatorname{sech}^2(\rho_t Z)]\to1$.

For finite $\rho_t$, the scalar Stein identity gives
\begin{align}
    \left\|\Sigma_t^{-1/2}\zeta_t\right\|_2
    &=
    \rho_t\,
    \mathbb E_Z\!\left[\operatorname{sech}^2(\rho_t Z)\right]
    \nonumber\\
    &=
    \mathbb E_Z[Z\tanh(\rho_t Z)].
    \label{eq:finite_update_magnitude}
\end{align}
Let $g(\rho)=\mathbb E_Z[Z\tanh(\rho Z)]$. Since
\begin{equation}
    g'(\rho)
    =
    \mathbb E_Z\!\left[Z^2\operatorname{sech}^2(\rho Z)\right]
    >0
\end{equation}
for finite $\rho\geq0$, the correction magnitude increases continuously with $\rho_t$. Moreover, dominated convergence gives
\begin{equation}
    \lim_{\rho\to\infty}g(\rho)
    =
    \mathbb E_Z|Z|
    =
    \sqrt{\frac{2}{\pi}}.
\end{equation}
Consequently,
\begin{equation}
    \left\|\Sigma_t^{-1/2}\zeta_t\right\|_2
    \leq
    \sqrt{\frac{2}{\pi}}
    \approx0.798.
    \label{eq:finite_update_bound}
\end{equation}
The deviation from the local result thus grows smoothly and remains bounded.

\paragraph{Step 4: A sufficient alignment condition.}
Let
\begin{equation}
    u_t=\Sigma_t^{-1/2}\Delta\mu_t^\star,
    \qquad
    w_t=\Sigma_t^{-1/2}\zeta_t.
\end{equation}
The inner product between the oracle direction and the bracketed finite-update direction in Eq.~\eqref{eq:finite_update_expected_direction}, both expressed in variance-normalized transition-mean space, satisfies
\begin{align}
    \left\langle u_t,\kappa_t u_t-w_t\right\rangle
    &\geq
    \kappa_t\|u_t\|_2^2-\|u_t\|_2\|w_t\|_2.
\end{align}
It is therefore positive whenever
\begin{equation}
    \left\|\Sigma_t^{-1/2}\zeta_t\right\|_2
    <
    \kappa_t
    \left\|\Sigma_t^{-1/2}\Delta\mu_t^\star\right\|_2.
    \label{eq:finite_update_positive_alignment}
\end{equation}
This is a sufficient, not necessary, condition. If it fails, the bound no longer certifies positive alignment, but it does not imply that the direction reverses.

\paragraph{Step 5: Effective update size.}
Substituting the definition of $d_t$ into Eq.~\eqref{eq:rho_zeta_definition} gives
\begin{equation}
    \rho_t
    =
    \beta
    \left\|
        \Sigma_t^{-1/2}B_t
        \left(v_{\theta,t}-v_t^{\mathrm{old}}\right)
    \right\|_2.
    \label{eq:effective_normalized_step}
\end{equation}
Thus, the accuracy of the local approximation is governed by the normalized displacement $\rho_t$, which jointly depends on $\beta$, the timestep-dependent operator $B_t$, the SDE covariance, and the actor--rollout-policy gap. The affine relation $\mu_t(v)=U_t+B_t v$ remains exact for any $\beta$; increasing $\rho_t$ instead strengthens the bounded finite-update correction in Eq.~\eqref{eq:finite_update_expected_direction}. This characterization is consistent with Figure~\ref{fig:hparam_beta}: $\beta\in\{0.5,1,2\}$ maintains strong performance, while $\beta=5$ degrades performance without complete collapse. The EMA rollout update and trust-region regularization control $v_{\theta,t}-v_t^{\mathrm{old}}$ and therefore help keep $\rho_t$ small.

\subsection{Proof of Lemma~\ref{lem:wmse_ranking_decomposition}}
\label{app:proof_Comparison with wMSE}

From Eq.~\eqref{eq:constructed_mu_pm}, we have $\mu_{\theta,t}^{\pm}=\mu_t^{\mathrm{old}}\pm d_t$ and $E_{\theta,t}^{\pm}=\|e_t\mp d_t\|_{\Sigma_t^{-1}}^2$.

For the weighted-regression objective in Eq.~\eqref{eq:step_wmse},
\begin{align}
    \ell_t^{\mathrm{wMSE}}(\theta)
    &=
    r\|e_t-d_t\|_{\Sigma_t^{-1}}^2
    +(1-r)\|e_t+d_t\|_{\Sigma_t^{-1}}^2
    \nonumber\\
    &=
    \|e_t\|_{\Sigma_t^{-1}}^2
    +\|d_t\|_{\Sigma_t^{-1}}^2
    +2(1-2r)\langle\Sigma_t^{-1}e_t,d_t\rangle
    \nonumber\\
    &=
    \mathrm{const}
    -2y\langle\Sigma_t^{-1}e_t,d_t\rangle
    +\|d_t\|_{\Sigma_t^{-1}}^2,
\end{align}
where $y=2r-1$ and $\mathrm{const}=\|e_t\|_{\Sigma_t^{-1}}^2$ is independent of $\theta$. This proves part (a).

For the step-wise ranking loss, Eq.~\eqref{eq:appendix_error_difference} gives
\begin{equation}
    E_{\theta,t}^{+}-E_{\theta,t}^{-}
    =
    -4\langle\Sigma_t^{-1}e_t,d_t\rangle .
\end{equation}
Substituting into Eq.~\eqref{eq:step_ranking},
\begin{equation}
    \ell_t(\theta)
    =
    \operatorname{softplus}
    \left(
    \frac12 y(E_{\theta,t}^{+}-E_{\theta,t}^{-})
    \right)
    =
    \operatorname{softplus}
    \left(
    -2y\langle\Sigma_t^{-1}e_t,d_t\rangle
    \right),
\end{equation}
which proves part (b). \hfill$\square$

\paragraph{Remark: binary feedback.}
For $r=1$, $\ell_t^{\mathrm{wMSE}}=E^+_{\theta,t}$; for $r=0$, $\ell_t^{\mathrm{wMSE}}=E^-_{\theta,t}$. Thus each binary-labeled sample explicitly optimizes its outcome-selected absolute error, while ranking uses the relative error $E^+_{\theta,t}-E^-_{\theta,t}$. Both objectives use success and failure samples, and both mirrored outputs depend on the same trainable field $v_\theta$, with parameter derivatives $\partial_\theta v^\pm_{\theta,t}=\pm\beta\partial_\theta v_{\theta,t}$. Changes to $v_\theta$ therefore affect both recomputed branches.

\newpage
\section{Experiment Details}
\label{app:experiment_details}
\subsection{Detailed Introduction of Benchmarks}
\label{app:Benchmarks}
We evaluate on two embodied-control benchmark families: LIBERO's four multitask suites and ManiSkill's large compositional pick-and-place task family.
\begin{itemize}[leftmargin=1.2em]
    \item LIBERO~\citep{libero}: We follow the standard protocol across four suites (Spatial, Object, Goal, Long), reporting average success rates over 500 episodes (50 states $\times$ 10 sub-tasks) per suite. The agent receives dual 224$\times$224 RGB inputs, language instructions, and 7-dimensional proprioceptive states (6-DoF joints + gripper). It outputs continuous end-effector actions. The environment provides a sparse binary reward (1 for success, 0 otherwise).
    \item ManiSkill~\citep{maniskill}: We adopt the PutOnPlateInScene multitask setting from RL4VLA~\citep{rl4vla}. This benchmark defines 4,352 compositional tasks derived from 16 objects, 17 receptacles, and 16 tabletop scenes. Observations consist of a 480$\times$640 third-person view, language instructions, and joint poses. Actions are continuous joint-space commands. The environment provides a composite reward to discourage degenerate throwing behaviors. For $\pi$-StepNFT, we use only trajectory-level binary success as the optimality signal, without shaped reward components such as the auxiliary $0.1$ grasp bonus.
\end{itemize}

\begin{table}[h]
\centering
\caption{OOD task mapping for ManiSkill \texttt{PutOnPlateInScene25*} across Vision, Semantics, and Execution categories.}
\label{tab:maniskill-ood-mapping}
\small
\renewcommand{\arraystretch}{1.1} 
\setlength{\tabcolsep}{4pt}
\begin{tabular}{@{}>{\raggedright\arraybackslash}p{1.5cm}
                >{\raggedright\arraybackslash}p{5.5cm}
                >{\raggedright\arraybackslash}p{6.35cm}@{}}
\toprule
\textbf{Category} & \textbf{Sub-category (OOD type)} & \textbf{ManiSkill env IDs} \\
\midrule
\multirow{5}{*}{\textbf{Vision}} 
& Unseen Table (background) 
& \texttt{VisionImage-v1} \\
& Dynamic Textures (foreground, weak) 
& \texttt{VisionTexture03-v1} \\
& Dynamic Textures (foreground, strong) 
& \texttt{VisionTexture05-v1} \\
& Dynamic Noise (image-level, weak) 
& \texttt{VisionWhole03-v1} \\
& Dynamic Noise (image-level, strong) 
& \texttt{VisionWhole05-v1} \\
\midrule
\multirow{5}{*}{\textbf{Semantics}} 
& Unseen Objects
& \texttt{Carrot-v1} \\
& Unseen Receptacles 
& \texttt{Plate-v1} \\
& Unseen Instruction Phrasings 
& \texttt{Instruct-v1} \\
& Multi-Object 
& \texttt{MultiCarrot-v1} \\
& Distractive Receptacle  
& \texttt{MultiPlate-v1} \\
\midrule
\multirow{3}{*}{\textbf{Execution}} 
& Unseen Position
& \texttt{Position-v1} \\
& Unseen Robot Init Pose 
& \texttt{EEPose-v1} \\
& Mid-Episode Object Reposition 
& \texttt{PositionChangeTo-v1} \\
\bottomrule
\end{tabular}
\end{table}

To study how long-horizon tasks can benefit when subgoal-level feedback is available, we additionally include a controlled diagnostic.

\begin{itemize}[leftmargin=1.2em]
    \item CALVIN~\citep{calvin}: We use CALVIN as a controlled long-horizon diagnostic benchmark with subgoal-level binary feedback. CALVIN evaluates language-conditioned tabletop manipulation, where a single policy must execute multiple natural-language instructions in sequence without resetting between subtasks. Each task contains up to 5 consecutive language-instructed subtasks, with a maximum horizon of 360 environment steps per subtask. For each subtask, the policy keeps attempting the current instruction until it succeeds or reaches the maximum horizon, and then proceeds to the next instruction only after success. This provides a binary completion signal for each subtask, making CALVIN suitable for studying whether intermediate subgoal success/failure feedback can stabilize long-horizon online fine-tuning. 
\end{itemize}

\subsection{SFT Initialization}
\label{app:sft_initialization}

We initialize each policy from the corresponding $\pi_\texttt{RL}$ SFT checkpoint~\citep{pirl}. For LIBERO, we use few-shot SFT initializations to prevent performance saturation from masking online RL gains. Specifically, $\pi_0$ is initialized from SFT checkpoints trained on pruned subsets of the total 1,692 LIBERO trajectories: 58 trajectories for Spatial, Object, and Goal, and 208 trajectories for Long. For $\pi_{0.5}$, we use a unified few-shot set of 40 trajectories, corresponding to one demonstration per subtask. For ManiSkill, we use the full 16,384 demonstration trajectories because the PutOnPlateInScene multitask benchmark has higher task complexity.

This initialization protocol keeps LIBERO in a partially solved regime where online RL can expose additional policy improvement, while using a stronger supervised initialization for ManiSkill where sparse multitask generalization is substantially harder.

\begin{table}[h]
\centering
\caption{Existing assets used in the experiments. ``Not specified'' means that the public asset page did not expose separate license metadata.}
\label{tab:asset_licenses}
\scriptsize
\setlength{\tabcolsep}{4pt}
\renewcommand{\arraystretch}{1.15}
\resizebox{\textwidth}{!}{%
\begin{tabular}{p{1.8cm} p{3.8cm} p{2.8cm} p{4cm} p{4.6cm}}
\toprule
\textbf{Asset} & \textbf{Use} & \textbf{Version / ID} & \textbf{License / terms} & \textbf{URL} \\
\midrule
LIBERO & Benchmark  & Spatial/Object/Goal/Long suites & Code: MIT\newline Data: CC BY 4.0 & \url{https://github.com/Lifelong-Robot-Learning/LIBERO} \\
ManiSkill & Benchmark & ManiSkill 3\newline release v3.0.1 & Apache-2.0 & \url{https://github.com/haosulab/ManiSkill} \\
CALVIN & Long-horizon diagnostic & Public repository & MIT & \url{https://github.com/mees/calvin} \\
OpenPI $\pi_0$ / $\pi_{0.5}$ & Base VLA policies & \texttt{pi0\_base}\newline \texttt{pi05\_base} & Code: Apache-2.0\newline Gemma terms for Gemma components & \url{https://github.com/Physical-Intelligence/openpi} \\
RLinf / $\pi_\texttt{RL}$ & RL framework and SFT checkpoints & RLinf v0.2\newline Hugging Face checkpoints & Code: Apache-2.0\newline \texttt{Pi05-LIBERO}: Apache-2.0\newline ManiSkill SFT: MIT\newline \texttt{Pi0-LIBERO}: not specified & \url{https://github.com/RLinf/RLinf}\newline \url{https://huggingface.co/RLinf} \\
\bottomrule
\end{tabular}
}
\end{table}

\subsection{Ablation Protocol and Compute}
\label{app:ablation_protocol}

For the component and hyperparameter ablations in Sections~\ref{sec:diagnostic}--\ref{sec:sensitivity}, each setting is run with five random seeds. Curves show the mean, and shading denotes $\pm1$ standard deviation across seeds. We use a lightweight ablation configuration to enable rapid iteration: each rollout iteration collects 32 interaction trajectories, uses batch size 256, and sets the actor learning rate to $5.0\times10^{-6}$. These ablations were conducted on $8\times$ NVIDIA RTX 4090 (48GB) GPUs, with one ablation run taking approximately 2.5 hours.

We choose the ablation task by evaluating the few-shot SFT initialization from Section~\ref{app:sft_initialization} on all LIBERO Object subtasks. As shown in Table~\ref{tab:libero_object_initial_success}, post-SFT success varies substantially across subtasks. We therefore use LIBERO Object Task 0, \textit{pick up the alphabet soup and place it in the basket}, as the default ablation task: its unsaturated initial success leaves room for online improvement while avoiding a near-zero starting point that would destabilize quick ablations.

\begin{table}[h]
\centering
\caption{Initial few-shot SFT success rates on LIBERO Object subtasks used to select the ablation task.}
\label{tab:libero_object_initial_success}
\small
\setlength{\tabcolsep}{5pt}
\renewcommand{\arraystretch}{1.08}
\begin{tabular}{@{}c p{8.4cm} c@{}}
\toprule
\textbf{ID} & \textbf{LIBERO Object task} & \textbf{Success rate} \\
\midrule
0 & pick up the alphabet soup and place it in the basket & 36.0\% \\
1 & pick up the cream cheese and place it in the basket & 85.0\% \\
2 & pick up the salad dressing and place it in the basket & 97.0\% \\
3 & pick up the bbq sauce and place it in the basket & 72.0\% \\
4 & pick up the ketchup and place it in the basket & 90.0\% \\
5 & pick up the tomato sauce and place it in the basket & 79.0\% \\
6 & pick up the butter and place it in the basket & 94.0\% \\
7 & pick up the milk and place it in the basket & 69.0\% \\
8 & pick up the chocolate pudding and place it in the basket & 54.0\% \\
9 & pick up the orange juice and place it in the basket & 16.0\% \\
\midrule
\multicolumn{2}{r}{\textbf{Average}} & 69.2\% \\
\bottomrule
\end{tabular}
\end{table}

To isolate the effect of subgoal feedback from the confounding factors introduced by multi-task training, we use a single CALVIN long-horizon sequence for the subgoal-feedback analysis. Specifically, we select trial ID 0, which consists of the following five-subgoal instruction chain: \textit{take the blue block and rotate it to the right} $\rightarrow$ \textit{push the sliding door to the right side} $\rightarrow$ \textit{lift the red block from the sliding cabinet} $\rightarrow$ \textit{store the grasped block in the sliding cabinet} $\rightarrow$ \textit{use the switch to turn off the light bulb}. We analyze task completion by tracking both the average number of completed subtasks per rollout and the success rate at each sequence length, which together reveal whether a method merely improves early-stage subtasks or consistently advances further along the long-horizon chain. For this analysis, each rollout iteration collects 64 interaction trajectories and uses a batch size of 512. We train for 100 iterations, with the full run taking approximately 6.5 hours.

Main experiments train on full benchmark suites with longer horizons, so we report results averaged over three seeds. They were conducted on $8\times$ NVIDIA H100 (80GB) GPUs.

\subsection{Hyperparameters for Training}
\label{app:hyperparameters}

The table below reports rollout horizons, parallel environments, rollout epochs, denoising steps, and actor-update settings. SFT checkpoint sources and demonstration counts are given in Appendix~\ref{app:sft_initialization}.

\paragraph{Feedback construction.}
The default optimality score on LIBERO and ManiSkill is trajectory-level binary success. For soft feedback, let $A_i$ denote the GRPO or GAE advantage. Following Diffusion-NFT~\citep{diffusionnft}, we map it to the bounded optimality score
\begin{equation}
    r_i=\frac12+\frac12\operatorname{clip}\!\left(\frac{A_i}{Z},-1,1\right),
    \qquad y_i=2r_i-1.
    \label{eq:feedback_normalization}
\end{equation}
Here $Z>0$ is a fixed scale hyperparameter. We use the same $Z=3$ for GRPO- and GAE-based feedback in all experiments. The $\pi$-StepNFT + Critic variant computes $A_i$ using GAE from a PPO critic while keeping the $\pi$-StepNFT actor objective unchanged. In Figure~\ref{fig:ablation_credit}, Norm. GAE and Norm. GAE + wMSE use identical scores $r_i$ and differ only in the actor objective.

\begin{table}[htbp]
\centering
\small

\caption{Hyperparameter settings for Libero and ManiSkill.}

\resizebox{\textwidth}{!}{
    \begin{tabular}{lccccccccccccc} 
    \toprule
    \multirow{4}{*}{\textbf{Parameters}} & \multicolumn{9}{c}{\textbf{LIBERO}} & & \multicolumn{3}{c}{\textbf{ManiSkill}} \\
    \cmidrule{2-10} \cmidrule{12-14}
     & \multicolumn{4}{c}{$\pi_0$} & & \multicolumn{4}{c}{$\pi_{0.5}$} & & $\pi_0$ & & $\pi_{0.5}$ \\
    \cmidrule{2-5} \cmidrule{7-10} \cmidrule{12-12} \cmidrule{14-14}
     & \textbf{Spatial} & \textbf{Object} & \textbf{Goal} & \textbf{Long} & & \textbf{Spatial} & \textbf{Object} & \textbf{Goal} & \textbf{Long} & & \textbf{Multitask} & & \textbf{Multitask} \\
    \midrule
    Train epochs & 400 & 400 & 400 & 400 & & 400 & 400 & 400 & 400 & & 240 & & 240 \\
    Batch size & 2048 & 2048 & 2048 & 2048 & & 2048 & 2048 & 2048 & 2048 & & 5120 & & 5120 \\
    Update epochs & 2 & 2 & 4 & 4 & & 1 & 1 & 3 & 4 & & 5 & & 5 \\
    Actor lr & 1e-5 & 1e-5 & 1e-5 & 1e-5 & & 1e-5 & 1e-5 & 1e-5 & 1e-5 & & 8e-6 & & 8e-6 \\
    \midrule
    Interaction steps & 240 & 240 & 320 & 480 & & 240 & 240 & 320 & 480 & & 60 & & 60 \\
    Parallel environments & 64 & 64 & 64 & 64 & & 64 & 64 & 64 & 64 & & 320 & & 320 \\
    Rollout epochs & 8 & 8 & 8 & 8 & & 8 & 8 & 8 & 8 & & 1 & & 1 \\
    \midrule
    Action chunk $H$ & 5 & 5 & 5 & 10 & & 5 & 5 & 5 & 10 & & 5 & & 5 \\
    Denoise steps & 4 & 4 & 4 & 4 & & 4 & 4 & 4 & 4 & & 4 & & 4 \\
    Noise level $\sigma$ & 0.2 & 0.2 & 0.2 & 0.2 & & 0.2 & 0.2 & 0.2 & 0.2 & & 0.2 & & 0.2 \\
    Branch Displacement Scale $\beta$ & 1.0 & 1.0 & 1.0 & 1.0 & & 1.0 & 1.0 & 1.0 & 1.0 & & 1.0 & & 1.0 \\
    Initial Decay $\alpha_0$ & 0.1 & 0.1 & 0.1 & 0.1 & & 0.1 & 0.1 & 0.1 & 0.1 & & 0.1 & & 0.1 \\
    End Decay $\alpha_{-1}$ & 0.995 & 0.995 & 0.995 & 0.995 & & 0.995 & 0.995 & 0.995 & 0.995 & & 0.995 & & 0.995 \\
    \bottomrule
    \end{tabular}
} %
\end{table}

\clearpage
\section{Additional Results}
\label{app:additional_results}

\subsection{Noise Variance and Training Usage}
\label{app:noise_variance_usage}

We further test whether the gains of SDE exploration can be reproduced by a simpler alternative that injects Gaussian noise only at the final denoised action $x_0$. 
This comparison is important because, under our Flow-SDE schedule, the global noise level $\sigma_0=0.2$ does not imply that every transition receives standard deviation $0.2$. 
Instead, the actual injected standard deviation is $\sqrt{\delta_t}\sigma_t$, where $\sigma_t$ follows the solver schedule. 
With four denoising steps, this gives progressively smaller transition noise and the last step has standard deviation about $0.05$. 
Thus, comparable noise added directly to $x_0$ might provide a lightweight substitute for SDE exploration.

\begin{table}[h]
    \centering
    \caption{Ablation results under different noise variances and noise-usage strategies.}
    \label{tab:noise_ablation}
    \resizebox{\columnwidth}{!}{%
    \begin{tabular}{c c c c c}
        \toprule
        \multirow{2}{*}{\textbf{Variance}} 
        & \multicolumn{2}{c}{\textbf{Noisy Actions Used Only for Environment Interaction}} 
        & \multicolumn{2}{c}{\textbf{Noisy Actions Also Used for Training}} \\
        \cmidrule(lr){2-3} \cmidrule(lr){4-5}
        & Decay $\alpha_0 = 0.1$ 
        & $\alpha_0 = 0.9$ 
        & Decay $\alpha_0 = 0.1$ 
        & $\alpha_0 = 0.9$ \\
        \midrule
        0.01  & fails at startup & fails at startup &  0.74 & 0.78 \\
        0.1   & - & - & fails at startup & 0.60 \\
        0.2   & - & - & fails at startup & fails at startup \\
        \bottomrule
    \end{tabular}
    }
\end{table}

Table~\ref{tab:noise_ablation} shows that this is not the case. 
When noisy actions are used only for environment interaction while training still targets the original $x_0$, the policy collapses immediately, consistent with noisy $x_0$ perturbations producing executions poorly matched by the training signal.
When the same noisy actions are also used for training with a variance-normalized correction, performance improves only under very small noise, reaching around $0.74$--$0.78$, which is comparable to deterministic ODE training but far below the step-wise target results in Figure~\ref{fig:ablation_target}. 
As the injected noise increases, training again becomes unstable or collapses.

These results suggest that simply perturbing the final action is insufficient in this diagnostic, and that denoising-path noise should be paired with supervision on the corresponding local transitions.
Adding noise only to the executed action does not directly supervise adjacent solver transitions, whereas our step-wise formulation fits the stochastic transition that generated the explored state.

\subsection{Training-Time System Profiling}
We profile peak allocated memory and wall-clock runtime on 8$\times$ NVIDIA A100 (80GB) GPUs. As shown in Table~\ref{tab:runtime_profiling}, $\pi$-StepNFT uses 18,822.0 MiB of peak memory, 5.7\% more than Flow-SDE with its lightweight critic (17,812.7 MiB), but 51.5\% less than Flow-Noise (38,819.3 MiB), which retains the full solver chain. Its global training iteration takes 286 seconds, compared with 310 seconds for Flow-SDE and 354 seconds for Flow-Noise, corresponding to reductions of 7.7\% and 19.2\%, respectively. The saving over Flow-Noise is concentrated in actor training, where $\pi$-StepNFT updates one sampled solver transition instead of backpropagating through the full multi-step solver chain, reducing actor time from 196 to 146 seconds. By contrast, its actor-training cost is nearly identical to Flow-SDE (146 versus 147 seconds), and the smaller overall time reduction is primarily observed during rollout generation (132 versus 151 seconds).

\begin{table}[h]
\centering
\caption{
System profiling per global training iteration on 8$\times$ NVIDIA A100 (80GB) GPUs.
Peak allocated memory is reported in MiB and stage times in seconds. Overall iteration time includes rollout generation, actor training, weight synchronization, and other minor overheads.
}
\label{tab:runtime_profiling}
\small
\setlength{\tabcolsep}{7pt}
\begin{tabular}{lccc}
\toprule
Metric & $\pi$-StepNFT & Flow-SDE & Flow-Noise \\
\midrule
Peak allocated memory (MiB) & 18,822.0 & 17,812.7 & 38,819.3 \\
Rollout generation (s) & 132 & 151 & 149 \\
Actor training (s) & 146 & 147 & 196 \\
Overall iteration (s) & 286 & 310 & 354 \\
\bottomrule
\end{tabular}
\end{table}

\subsection{Full ManiSkill OOD Results}
\label{app:maniskill_ood_results}

For ManiSkill, we follow the PutOnPlateInScene OOD taxonomy used by RL4VLA and the $\pi_\texttt{RL}$ grouping convention, placing instruction-phrasing shifts together with visual perturbations under Visual-Language Variations, since both preserve the underlying manipulation objective while changing the observation/language surface form. Our primary ManiSkill baselines are the $\pi_\texttt{RL}$ Flow-SDE and Flow-Noise variants. To keep the comparison protocol conservative and traceable, we use the $\pi_\texttt{RL}$ results reported by the original authors for the variants covered in their evaluation, rather than mixing reported baseline numbers with separately re-run numbers in the same table. Consequently, for additional RL4VLA-style variants that are not reported in $\pi_\texttt{RL}$, including Carrot-v1, Plate-v1, and EEPose-v1, we mark the corresponding $\pi_\texttt{RL}$ entries as unavailable and report $\pi$-StepNFT only for completeness of our evaluation coverage. All direct category averages and comparisons are computed only over the shared variants for which results are available for every compared method.

\begin{table}[h]
\centering
\caption{Success rates (\%) on ManiSkill for $\pi_0$ and $\pi_{0.5}$ across in-distribution and OOD variants.}
\scriptsize
\setlength{\tabcolsep}{2.5pt}
\renewcommand{\arraystretch}{1.05}
\resizebox{\textwidth}{!}{%
\begin{tabular}{llcccccccc}
\toprule
\textbf{Environment} & \textbf{Variation-Version-Type}
& $\pi_0$-SFT
& \begin{tabular}[c]{@{}c@{}}$\pi_0$-RL\\Flow-SDE\end{tabular}
& \begin{tabular}[c]{@{}c@{}}$\pi_0$-RL\\Flow-Noise\end{tabular}
& $\pi_0$-StepNFT
& $\pi_{0.5}$-SFT
& \begin{tabular}[c]{@{}c@{}}$\pi_{0.5}$-RL\\Flow-SDE\end{tabular}
& \begin{tabular}[c]{@{}c@{}}$\pi_{0.5}$-RL\\Flow-Noise\end{tabular}
& $\pi_{0.5}$-StepNFT \\
\midrule

\multirow{1}{*}{In distribution}
& Main-v3-train
& 38.4 & \textbf{78.8} & 77.8 & 73.6
& 40.1 & \textbf{90.9} & 89.7 & 84.3 \\
\midrule

\multirow{8}{*}{\begin{tabular}[c]{@{}l@{}}Visual-Language\\Variations\end{tabular}}
& Instruct-v1-test
& 30.1 & 64.6 & \textbf{66.5} & 62.5
& 46.6 & 77.0 & \textbf{85.7} & 82.9 \\
& VisionImage-v1-test
& 38.3 & 68.8 & 71.7 & \textbf{71.8}
& 46.2 & 78.8 & \textbf{83.1} & 80.5 \\
& VisionTexture03-v1-test
& 35.1 & 66.0 & 66.8 & \textbf{72.7}
& 36.7 & 69.6 & 75.0 & \textbf{80.5} \\
& VisionTexture05-v1-test
& 31.0 & 55.8 & 60.5 & \textbf{63.1}
& 32.7 & 58.0 & 62.2 & \textbf{71.2} \\
& VisionWhole03-v1-test
& 35.4 & 62.4 & 69.0 & \textbf{70.2}
& 40.1 & 69.6 & 71.6 & \textbf{82.0} \\
& VisionWhole05-v1-test
& 28.5 & 49.0 & 53.9 & \textbf{54.7}
& 30.7 & 55.0 & 57.0 & \textbf{62.5} \\
\cmidrule{2-10}
& \textit{Visual Avg. (5)}
& 33.7 & 60.4 & 64.4 & \textbf{66.5}
& 37.3 & 66.2 & 69.8 & \textbf{75.3} \\
& \textit{Visual-Language Avg. (6)}
& 33.1 & 61.1 & 64.7 & \textbf{65.8}
& 38.8 & 68.0 & 72.4 & \textbf{76.6} \\
\midrule

\multirow{3}{*}{\begin{tabular}[c]{@{}l@{}}Semantic Reasoning\\(object/receptacle\\confounders)\end{tabular}}
& MultiCarrot-v1-test
& 7.8 & \textbf{28.2} & 23.0 & 26.0
& 16.7 & 36.8 & 38.2 & \textbf{39.1} \\
& MultiPlate-v1-test
& 5.0 & 16.4 & \textbf{18.3} & 17.6
& 11.8 & \textbf{29.4} & 28.3 & 21.9 \\
& Carrot-v1-test
& 23.5 & - & - & 58.3
& 31.2 & - & - & 65.2 \\
& Plate-v1-test
& 38.0 & - & - & 72.7
& 39.1 & - & - & 78.9 \\

\cmidrule{2-10}
& \textit{Avg. over shared variants}
& 6.4 & \textbf{22.3} & 20.7 & 21.8
& 14.3 & 33.1 & \textbf{33.3} & 30.5 \\

\midrule

\multirow{3}{*}{Action Execution}
& PositionChangeTo-v1-test
& 9.6 & 17.4 & 10.9 & \textbf{19.5}
& 13.5 & 36.2 & \textbf{54.7} & 28.1 \\
& Position-v1-test
& 16.9 & \textbf{45.6} & 37.5 & 39.8
& 31.2 & 54.5 & 55.0 & \textbf{56.3} \\
& EEPose-v1-test 
& 11.7 & - & - & 29.3 
& 18.6 & - & - & 49.1 \\
\cmidrule{2-10}
& \textit{Avg. over shared variants}
& 13.3 & \textbf{31.5} & 24.2 & 29.7
& 22.4 & 45.4 & \textbf{54.9} & 42.2 \\
\bottomrule
\end{tabular}
}
\label{tab:maniskill_full}
\end{table}

\section{Limitations}
\label{app:limitations}
Online RL for large flow-based VLAs still requires many environment interactions, while real-robot rollouts are slow and difficult to parallelize. The sampling efficiency of the current method therefore does not yet support direct online training on real hardware. Improving interaction efficiency and validating the resulting method on physical robots remain important directions for future work.

\end{document}